\documentclass{article}

\PassOptionsToPackage{numbers, compress}{natbib}
\usepackage[preprint ]{neurips_2024}

\usepackage[utf8]{inputenc} % allow utf-8 input
\usepackage[T1]{fontenc}    % use 8-bit T1 fonts
\usepackage{hyperref}       % hyperlinks
\usepackage{url}            % simple URL typesetting
\usepackage{booktabs}       % professional-quality tables
\usepackage{amsfonts}       % blackboard math symbols
\usepackage{nicefrac}       % compact symbols for 1/2, etc.
\usepackage{microtype}      % microtypography
\usepackage{xcolor}         % colors

\usepackage{amsmath,amsfonts,bm}

\def\ourmethod{SEER\xspace}
\def\PbRL{preference-based RL\xspace}
\def\st{s_t}
\def\at{a_t}

\def\stt{s^\prime}
\def\att{a^\prime}
\def\empiricalQ{\widehat{Q}}

\def\eqref#1{equation~\ref{#1}}
\def\1{\bm{1}}

\DeclareMathAlphabet{\mathsfit}{\encodingdefault}{\sfdefault}{m}{sl}
\SetMathAlphabet{\mathsfit}{bold}{\encodingdefault}{\sfdefault}{bx}{n}

\def\gA{{\mathcal{A}}}
\def\gB{{\mathcal{B}}}

\def\gD{{\mathcal{D}}}
\def\gE{{\mathcal{E}}}

\def\gG{{\mathcal{G}}}

\def\gJ{{\mathcal{J}}}

\def\gL{{\mathcal{L}}}

\def\gN{{\mathcal{N}}}
\def\gO{{\mathcal{O}}}
\def\gP{{\mathcal{P}}}

\def\gR{{\mathcal{R}}}
\def\gS{{\mathcal{S}}}
\def\gT{{\mathcal{T}}}

\def\gV{{\mathcal{V}}}

\newcommand{\E}{\mathbb{E}}

\newcommand{\KL}{D_{\mathrm{KL}}}

\usepackage{xspace}
\usepackage{algorithm}
\usepackage{algorithmic}
\usepackage{enumitem} % [leftmargin=20pt]
\usepackage{subfig} % subfloat
\usepackage{multirow}
\usepackage{graphicx}
\usepackage{booktabs}
\usepackage[capitalize,noabbrev]{cleveref}
\usepackage{parskip}
\usepackage{array}
\usepackage{arydshln}
\newcommand{\stdv}[1]{\scalebox{.70}{~$\pm$~#1}}

\makeatletter
\def\myscriptsize{\@setfontsize\scriptsize{8.3pt}{9.5pt}}
\makeatother

\usepackage{amsmath}
\usepackage{amssymb}
\usepackage{mathtools}
\usepackage{amsthm}
\theoremstyle{plain}
\newtheorem{theorem}{Theorem}[section]

\newtheorem{lemma}[theorem]{Lemma}

\theoremstyle{definition}

\theoremstyle{remark}

\title{Efficient Preference-based Reinforcement Learning via Aligned Experience Estimation}

\author{%
  Fengshuo Bai$^{1,2,3}$
  Rui Zhao$^{4,\dag}$\thanks{$^{\dag}$Corresponding authors, contact Lei Han$<$\textit{lxhan@tencent.com}$>$.} \,
  Hongming Zhang$^{5}$
  Sijia Cui$^{6}$\\
  \textbf{Ying Wen$^{1}$ Yaodong Yang$^{2,\dag}$ Bo Xu$^{6}$ Lei Han$^{4, \dag}$}\\
  $^{1}$Shanghai Jiao Tong University \\
  $^{2}$Institute for Artificial Intelligence, Peking University \\
  $^{3}$National Key Laboratory of General Artificial Intelligence, BIGAI\\
  $^{4}$Tencent Robotics X $^{5}$University of Alberta\\
  $^{6}$Institute of Automation, Chinese Academy of Sciences\\
}

\begin{document}

\maketitle

\begin{abstract}
  Preference-based reinforcement learning (PbRL) has shown impressive capabilities in training agents without reward engineering. However, a notable limitation of PbRL is its dependency on substantial human feedback. This dependency stems from the learning loop, which entails accurate reward learning compounded with value/policy learning, necessitating a considerable number of samples. To boost the learning loop, we propose SEER, an efficient PbRL method that integrates label smoothing and policy regularization techniques. Label smoothing reduces overfitting of the reward model by smoothing human preference labels. Additionally, we bootstrap a conservative estimate $\widehat{Q}$ using well-supported state-action pairs from the current replay memory to mitigate overestimation bias and utilize it for policy learning regularization. Our experimental results across a variety of complex tasks, both in online and offline settings, demonstrate that our approach improves feedback efficiency, outperforming state-of-the-art methods by a large margin. Ablation studies further reveal that SEER achieves a more accurate Q-function compared to prior work.
\end{abstract}

\section{Introduction}
Deep Reinforcement Learning (DRL) has recently demonstrated remarkable proficiency in enabling agents to excel in complex behaviors across diverse domains, including robotic control and manipulation~\citep{lillicrap2015continuous, 7989385}, game playing~\citep{Dmnih2013playing, VinyalsBCMDCCPE19}, and industrial applications~\citep{xu2023drl}. The foundation of success lies in providing a well-designed reward function. However, setting up a suitable reward function has been challenging for many reinforcement learning applications~\citep{8202141, 7989307}. The quality of the reward function depends heavily on the designer's understanding of the core logic behind the problem and relevant background knowledge. For example, formulating a reward function for text generation presents a significant challenge due to the inherent difficulty in quantifying text quality on a numerical scale~\citep{abs-2109-10862, NEURIPS2022_b1efde53}. Despite the substantial efforts of expert engineers in reward engineering, previous research~\citep{dp_blog_2020, NEURIPS2022_3d719fee} has highlighted various challenges, such as ``reward hacking''. In these scenarios, agents focus solely on maximizing their rewards by exploiting misspecification in the reward function, often leading to unintended and potentially problematic behaviors.

Recently, PbRL has gained widespread attention and has fruitful outcomes~\citep{lee2021pebble, park2022surf, NEURIPS2022_8be9c134}. Rather than relying on hand-engineered reward functions, humans can provide preference between a pair of agent trajectories, thereby implicitly indicating the desired behaviors or the task's objectives. Recent research has demonstrated that PbRL can train agents to perform novel behaviors and mitigate the challenges of reward hacking to some extent. However, existing methods still suffer from feedback inefficiency, limiting the applicability of PbRL in practical scenarios. Taking a closer look at PbRL, it involves collecting preference labels, learning a reward model from preferences, optimizing policy with the reward model, and subsequently generating higher-quality trajectories for the next iteration, thereby creating a virtuous circle. Prior research~\citep{abs-2305-15363, NEURIPS2022_8be9c134} observes that the intrinsic inefficiency of reward learning mechanisms results in an increase in feedback requirements for PbRL. Specifically, insufficient preference label leads to an imprecise reward model.
This inaccuracy may cause the Q function to be misled by the erroneous outputs of the reward model, resulting in suboptimal policy, a phenomenon often referred to as confirmation bias~\citep{pham2021meta_pseudo_labels}. Further, the data coverage in the replay memory is limited to a tiny subset of the whole state-action space. When combined with deep neural networks, the extrapolation of function approximation may erroneously overestimate out-of-distribution state-action pairs to have unrealistic values, and the errors will be back-propagated to previous states~\citep{fujimoto2019off, kumar2019stabilizing, levine2020offline}. Due to an inaccurate reward model compounded with overestimation bias, the Q function drives a suboptimal policy, which deteriorates the learning cycle and leads to poor performance.

In this work, we present \ourmethod, an efficient framework via aligned estimation from experience for preference-based reinforcement learning.
Our approach integrates two complementary techniques: \textit{label smoothing} and \textit{policy regularization}. In terms of label smoothing, we smooth human preference labels to mitigate overfitting during reward learning. During RL training, we estimate a conservative $\widehat{Q}$ using only the transitions in the replay memory, effectively reducing overestimation bias. This estimate is then employed to regularize the policy learning using Kullback-Leibler (KL) divergence. 
Our empirical evaluations, conducted on a range of complex tasks in both online and offline environments, demonstrate that \ourmethod markedly surpasses baselines. This advantage is particularly obvious when the available human feedback is limited, showcasing \ourmethod's effectiveness in leveraging limited data to achieve superior performance.

The key \textbf{contributions} of our work are as follows:
\textbf{(1)} We propose \ourmethod, a novel feedback-efficient preference-based RL algorithm that combines policy regularization with label smoothing, facilitating efficient and effective learning.
\textbf{(2)} Experiments demonstrate that our method outperforms other state-of-the-art PbRL methods and substantially improves feedback efficiency across a variety of complex tasks, both in online and offline settings.
\textbf{(3)} We demonstrate that, benefiting from policy regularization and label smoothing, our approach shows a clear advantage over PEBBLE, particularly in scenarios with limited human feedback. We also show that \ourmethod can train an accurate Q function and a better policy.

\section{Preliminaries}
\noindent \textbf{Preference-based Reinforcement Learning.}
In the RL paradigm, a finite Markov decision process (MDP) is defined by the tuple $\langle \gS, \gA, \gR, \gP, \gamma \rangle$. This comprises the state space $\gS$, action space $\gA$, transition dynamics, reward function $\gR$, and discount factor $\gamma$. The transition probability $\gP(s^\prime|s,a)$ characterizes the stochastic nature of the environment, indicating the probability of transitioning to state $s^\prime$ upon taking action $a$ in state $s$. $\gR(s,a)$ specifies the reward received for performing action $a$ in state $s$. The policy $\pi(a|s)$ maps the state space to the action space. The objective of the agent is to collect trajectories by interacting with the environment, aiming to maximize the expected return.

In the PbRL framework outlined by ~\citet{NIPS2017_d5e2c0ad}, the reward function from reward engineering is replaced by a reward function estimator, $\widehat{r}_{\psi}$, which is learned to align with human preferences. In this context, a segment $\sigma$ is defined as a sequence of states and actions, denoted as $(s_{t+1},a_{t+1},\cdots, s_{t+k},a_{t+k})$. A human provides a preference $y$ for a pair of segments $(\sigma^0,\sigma^1)$, where $y$ is a distribution over $\{(1,0),(0,1),(0.5,0.5)\}$. Following the Bradley-Terry model~\citep{pref-model-orig}, a preference predictor based on the estimated reward function $\widehat{r}_{\psi}$ is formulated as:
\begin{equation}
  P_{\psi}[\sigma^0\succ\sigma^1] = \text{Sigmoid} \Big(
  \sum_{t}\widehat{r}_{\psi}(s_t^0, a_t^0)-\sum_{t}\widehat{r}_{\psi}(s_t^1, a_t^1)
  \Big),
\label{eq:preference_r}
\end{equation}
where $\sigma^0 \succ \sigma^1$ indicates $\sigma^0$ is more consistent with human expert expectations than $\sigma^1$. The reward model is trained by minimizing the cross-entropy loss between predictions from the preference predictor and actual human preferences:
\begin{equation}
  \mathcal{L}_{\text{reward}}(\psi) = -\underset{(\sigma^0,\sigma^1,y)\sim \mathcal{D}}{\mathbb{E}} 
  \Big[ \sum_{i=0}^{1}y(i)\log P_\psi[\sigma^i\succ\sigma^{1-i}] \Big].
\label{eq:reward_loss}
\end{equation}
By optimizing $\widehat{r}_{\psi}$ with respect to this loss, segments that align more closely with human receive a higher return.

% =========================
\noindent \textbf{Twin Delayed DDPG.} 
TD3~\citep{pmlr-v80-fujimoto18a} introduces several key techniques for improvement. It employs the minimum value between two Q-networks to address the issue of overestimation resulting from function approximation errors. Delayed policy updates and target policy smoothing further refine the training process. TD3 learns two Q-functions, $Q_{\theta_1}$ and $Q_{\theta_2}$, by minimizing the mean square Bellman error. Concurrently, it optimizes the policy $\pi_\phi$ by maximizing the Q-function values.

For optimizing $Q_\theta$, it minimizes the temporal difference errors. The target value $y$ for the Q networks is defined as:
\begin{equation}
    \begin{aligned}
        y(\stt, &r) = r + \gamma\min\limits_{i=1,2} Q_{\bar\theta_i}(\stt,\att(\stt)),
        & \att(\stt) = \text{clip} \Big( \pi_{\bar\phi}(\stt) + \text{clip}(\epsilon, -c, c),a_{\text{L}},a_{\text{H}} \Big),
    \end{aligned}
\end{equation}
where $\bar\theta_1$, $\bar\theta_2$, and $\bar\phi$ denote the target networks for Q and policy networks, respectively. $\gamma$ is the discount factor, $\epsilon$ is clipped exploration noise from $\gN(0,\sigma^2)$ and clipped by threshold $c$, with $a_{\text{L}},a_{\text{H}}$ indicating the valid action range. Both Q-functions are trained by regressing to this target:
\begin{equation}
    J_Q(\theta_i) = \E_{(s,a,r,\stt)\sim\gD} \Big[ \left( Q_{\theta_i}(s,a)-y(\stt,r) \right)^2 \Big],
\label{eq:q_loss}
\end{equation}

In updating the policy, the objective is to learn a policy $\pi_\phi$ that maximizes $Q_{\theta_1}$. As the Q-function is differentiable with respect to action, gradient ascent can be applied:
\begin{equation}
    \begin{aligned}
        J_\pi(\phi) = \mathop{\E}\limits_{s \sim \gD}
        \big[ Q_\theta(s, \pi_\phi(s) ) \big], 
    \end{aligned}
\label{eq:pi_loss}
\end{equation}
where parameters $\theta$ are considered constant. 

In our work, we choose DQN~\citep{MnihKSGAWR13} and TD3 as the basic RL algorithms for discrete and continuous settings, respectively.

\begin{figure*}[t]
% \vspace{-0.1em}
\begin{center}
    \includegraphics[width=0.8\linewidth]{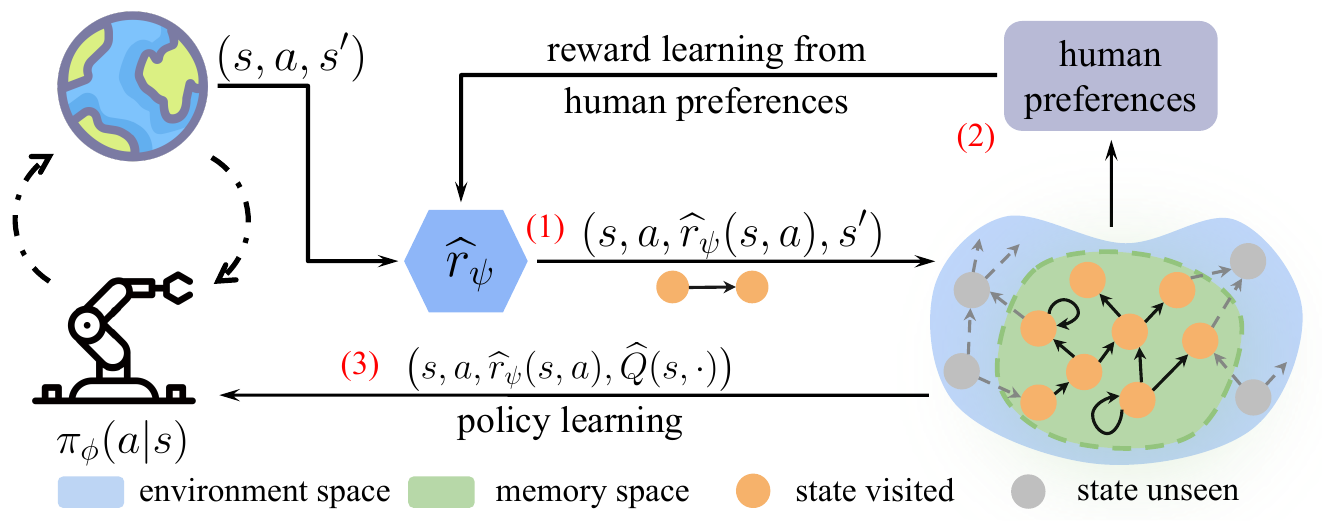}
    % \vspace{-0.5em}
    \caption{An illustration of \ourmethod. 
    (1) Label rewards using $\widehat{r}_{\psi}$.
    (2) Update the reward model $\widehat{r}_{\psi}$ with smoothed preference labels.
    (3) Estimate conservative $\widehat{Q}$ and regularize $Q_\theta$ and $\pi_\phi$.
    }
\label{fig:framework}
\end{center}
\vspace{-1.7em}
\end{figure*}

\section{Method}
In this section, we introduce \ourmethod, a generic framework designed to integrate with any PbRL approach, enhancing feedback efficiency. Figure \ref{fig:framework} illustrates the overview of our method. We will detail the three core components of \ourmethod: human label smoothing, conservative estimate $\widehat{Q}$ and policy regularization. The detailed procedure of our algorithm is outlined in Algorithm~\ref{pseudo_code} (online settings) and ~\ref{pseudo_code_offline} (offline settings).

\subsection{Reward Learning}
Revisiting the loss~(\ref{eq:reward_loss}) for optimizing the reward model in the framework of PbRL, it is a cross-entropy loss denoting the distance between the true distribution and predicted distribution. But when we assume all segment pairs in the preference dataset hold $\sigma^0 \succ \sigma^1$. In practice, human preference labels tend to exhibit strong polarization, often taking the form of $(1,0)$ or $(0,1)$. Consequently, the loss (\ref{eq:reward_loss}) simplifies to minimizing the negative logarithmic predicted probability of $\sigma^0 \succ \sigma^1$. The preference predictor~\ref{eq:preference_r} provided by the Bradley-Terry model~\citep{pref-model-orig} is a sigmoid function. Therefore, when we attempt to optimize the reward model such that $P_{\psi}[\sigma^0 \succ \sigma^1] = 1$, we are essentially hoping for the condition $ \left( \sum_{t}\widehat{r}_{\psi}(s_t^0, a_t^0) - \sum_{t}\widehat{r}_{\psi}(s_t^1, a_t^1) \right) \rightarrow \infty$ to be satisfied. However, this is infeasible and can negatively impact the performance of the reward model. Over-fitting becomes a significant empirical issue, particularly in cases where the action spaces are extremely large, such as with continuous action spaces. Then we smooth the human label according to the following rule:
\begin{equation}
    y(i,j) = \left\{
        \begin{array}{lr}
        (1-\lambda, \lambda), & \text{if } \sigma^i \succ \sigma^j \\
        (0.5, 0.5), & \text{otherwise.}
        \end{array}
    \right.
\end{equation}
By applying a smoothing technique to the human preference labels, we aim to achieve a more accurate reward model during the policy learning process.

\subsection{Conservative Estimate $\widehat{Q}$}
By utilizing well-supported state-action pairs from the current replay memory, we bootstrap a conservative estimate $\widehat{Q}$. This conservative estimation provides two main advantages: firstly, it enables further exploitation of the information in the replay memory; secondly, it prevents overestimation caused by extrapolation in unseen states and actions. We build the replay memory as a graph for discrete action spaces, which is similar to the method used by \citet{Zhu2020Episodic,hong2022topological}. This allows us to rapidly perform conservative estimations without significant additional computational overhead. For continuous action spaces, we employ a neural network $Q_\xi$ to model $\empiricalQ$.

\subsubsection{Discrete Setting}
We structure the replay memory as a dynamic and directed graph, which is denoted as $\mathcal{G}=(\mathcal{V}, \mathcal{E})$. Each vertex in this graph represents a state $s$ along with its associated action value estimation $\widehat{Q}(s,\cdot)$, forming the vertex set: $\mathcal{V}=\{s| (s,\widehat{Q}(s,\cdot)) \}$. Each directed edge in the graph denotes a transition from state $s$ to state $s^\prime$ through action $a$. These edges also store the estimated reward $\widehat{r}_{\psi}(s, a)$ and the count of transitions $N(s, a, s^\prime)$, which are crucial for updating the model. The set of graph edges is represented as $\gE = \{s \stackrel{a} {\rightarrow}s^\prime| (a, \widehat{r}_{\psi}(s, a), N(s, a, s^\prime), \widehat{Q}(s,a)\}$. To ensure efficient querying, each vertex and edge is assigned a unique key via a hash function, achieving a query time complexity of $\gO(1)$. Additionally, every vertex $v$ contains an action set $\partial\gA(s)$, which includes the actions executed in state $s$, thus facilitating in-sample updates. Similar to conventional replay memory, this graph stores the most recent experiences, maintaining a fixed memory size.

This graph updating includes two primary components: estimation updating and reward relabeling. Upon observing a transition $(s, a, \widehat{r}_{\psi}(s, a), s^\prime)$, we add a new vertex and edge following the previously outlined data structure, initializing $\widehat{Q}(s, a)=0$ and $N(s, a, s^\prime)=1$. If an edge for the transition already exists, the visit count is incremented: $N(s, a, s^\prime) \leftarrow N(s, a, s^\prime) + 1$. For updating the action value estimate $\widehat{Q}$, a subset of graph vertices $\partial \gV \subseteq \gV$ is sampled in reverse order, similar to techniques described in \citet{abs-1910-08780, NEURIPS2019_e6d8545d}, to enable rapid and efficient updates. During the update process, the max-operator in the update rule is constrained to operate over $\partial\gA(s)$ instead of the entire action space, preventing visits to out-of-sample actions. Specifically, we update $\widehat{Q}$ using value iteration, defined by:
\begin{equation}
    \widehat{Q}(s,a) \leftarrow 
    \sum_{\stt \in \gS} \widehat{p}(\stt|s,a) \Big[ 
    \widehat{r}_{\psi}(s, a) + \gamma  \max\limits_{\att \in \partial\gA(\stt)} \widehat{Q}(\stt,\att) \Big],
\label{eq:value_iteration}
\end{equation}
where $\widehat{p}(\stt|s,a) =  N(s, a, \stt)/\sum_{\stt} N(s, a, \stt)$ represents the empirical dynamics within the graph. This update rule (\cref{eq:value_iteration}) ensures that the method never queries values for unseen actions, thereby preventing overestimation. For reward relabeling, we relabel all past experiences using the reward model $\widehat{r}_{\psi}$ each time this reward model is updated. This technique maximizes the use of historical transitions and reduces the impact of a non-stationary reward function.

\subsubsection{Continuous Setting}
For tasks involving continuous actions, using discretization techniques to transform continuous action spaces into discrete ones may be beneficial. For example, RT-2~\citep{brohan2023rt} employs this approach for discretizing actions in robotic applications, and it has also been successfully applied in complex games like DOTA~\citep{berner2019dota} and StarCraft~\citep{VinyalsBCMDCCPE19}.To make our method more comprehensive, we explore methods for estimating $\widehat{Q}$ in continuous action spaces. To achieve that, we propose an operator analogous to the constrained $\max$ operator in formula~(\ref{eq:value_iteration}) for conservative Q estimation. Drawing from previous work~\citep{NEURIPS2020_0d2b2061, kostrikov2022offline, garg2023extreme}, we use the log-sum-exp defined as:
\begin{equation}
    \gT^\beta Q_\xi(s, a) = \beta\log \E_{(s,a)\sim\gD} \left[ \frac{1}{\beta} \exp( Q_\xi(s,a) ) \right],
\label{eq:logsumexp_Q}
\end{equation}
where $\beta$ is a scaling parameter. For any $\beta_1 > \beta_2$, it follows that $\gT^{\beta_1}(Q_\xi) < \gT^{\beta_2} Q_\xi$. Moreover, $\gT^\infty Q_\xi  = \E \left[Q_\xi\right]$, and $\gT^0 Q_\xi = \sup(Q_\xi)$. Therefore, for any $\beta \in (0, \infty)$, the operator $\gT^{\beta} Q_\xi $ interpolates between the expectation and the maximum of $Q_\xi$. And we optimize the $Q_\xi$ by minimizing the Mean-squared-error (MSE) loss, defined as:
\begin{equation}
    \gJ_Q(\xi) =  \mathop{\E}\limits_{\tau_t \sim \gD}
    \left[ \left(Q_\xi(s,a) - \widehat{r}_{\psi}(s_t,a_t) -\gamma \gT^\beta Q_\xi(s_{t+1})\right)^2\right],
\label{eq:loss_q_xi}
\end{equation}
where $\tau_t = ( s_t, a_t, s_{t+1}, \widehat{r}_{\psi}(s_t,a_t))$ is the transition.

\subsection{Policy Learning}
In the above section, we develop a strategy for obtaining a conservative estimate $\widehat{Q}$. We now propose a policy regularizer tailored for both discrete and continuous settings.

Given that the goal of learning is to derive an optimal policy from the Q network, we use policy regularization to align policy $\pi$ with the current best policy $\widehat{\pi}$ derived from $\widehat{Q}$. Specifically, we regularize policy $\pi$ by minimizing the Kullback-Leibler (KL) divergence between $\widehat{\pi}$ and $\pi$, defined as:
\begin{equation}
    \gL_{reg} = \mathop{\E}\limits_{s \sim \gD}
        \Big[ \KL \Big(\pi(s) \Vert \widehat{\pi}(s) \Big) \Big].
\label{eq:pi_regularize}
\end{equation}

In scenarios with discrete actions, we define $\widehat{\pi}$ as the Boltzmann policy derived from $\widehat{Q}$, where $\widehat{\pi}(s) = \mathop{\textrm{Softmax}}_{a \in \partial\gA(s)}(\widehat{Q}(s,\cdot))$. This policy is inherently conservative, considering only the support set $\partial\gA(s)$ for a given state $s$. Similarly, the policy $\pi$, derived from the $Q_\theta$ network, is expressed as $\pi(s) = \mathop{\textrm{Softmax}}_{a \in \partial\gA(s)}(Q_\theta(s,\cdot))$. Total loss is defined as follows:
\begin{equation}
    \begin{aligned}
        \gL_{\textrm{discrete}}(\theta) &= 
        \mathop{\E}\limits_{\tau_t \sim \gG} \Big[(Q_\theta(s,a) - y )^2 \Big] +
        \eta \gL_{\text{reg}}(\theta)
    \end{aligned}
\label{eq:discrete_q_loss}    
\end{equation}
where $y = \widehat{r}_{\psi}(s, a) + \gamma \max\limits_{\att}Q(\stt,\att)$ is the Q target, $\tau_t = ( s_t, a_t, s_{t+1}, \widehat{r}_{\psi}(s_t,a_t))$ is the transition and $\eta$ is the weight factor.

In scenarios with continuous actions, we need a parameterized policy $\pi_\phi$ to interact with the environment. For policy learning, $Q_\xi$ is employed to regularize the policy $\pi_\phi$, with the parameters $\phi$ optimized by maximizing the following objective:
\begin{equation}
    \begin{aligned}
        J_\pi(\phi) &= \mathop{\E}\limits_{s_t \sim \gD} 
        \big[Q_\theta(s_t, \pi_\phi(s_t) )\big] + \eta\gL_{\text{reg}}(\phi),
        &\gL_{\text{reg}}(\phi) = \KL \Big( \pi_\phi(s_t)  \Big\Vert \frac{\exp \left(Q_\xi(s_t,\cdot)\right)}{Z} \Big),
    \end{aligned}
\label{eq:continuous_pi_loss}
\end{equation}
where $\eta$ is the weight of the policy regularizer, the partition function $Z$ normalizes the distribution and the parameters $\theta$ and $\xi$ are considered constants.

\subsection{Theoretical Analysis}
We conduct a theoretical analysis to show the properties of the conservative estimate, denoted as $\widehat{Q}$, in a finite state-action space $\mathcal{S} \times \mathcal{A}$. To learn $\widehat{Q}$, we employ \cref{eq:value_iteration} based on the replay buffer. This approach bootstraps only in-distribution actions, yielding a conservative estimate of the action value. In comparison to the $Q$ function, which bootstraps from the entire action space, $\widehat{Q}$ reduces extrapolation error from out-of-distribution data. $\widehat{Q}$ serves as a lower bound for $Q$ and converges to the global optimum as data coverage expands. The complete proofs of \cref{theorem:bellman} are provided in Appendix~\ref{appendix:proof_theorem}.

\begin{theorem}
Consider the tabular case with finite state-action space $\mathcal{S}\times \mathcal{A}$. Let $Q_t$ and $\widehat{Q}_t$ represent the Q-values learned following the Bellman optimality equation and \cref{eq:value_iteration} at time step $t$, respectively. We have $Q_t$ and $\widehat{Q}_t$ converge to fixed points $Q^*$ and $\widehat{Q}^*$, i.e., $\lim_{t\rightarrow \infty} Q_t = Q^*$ and $\lim_{t\rightarrow \infty} \widehat{Q}_t = \widehat{Q}^*$. Furthermore, for all $(s,a) \in \mathcal{S} \times \mathcal{A}$, $\widehat{Q}^*(s,a) - Q^*(s,a)\le 0$. This equality holds if all state-action pairs are visited.
\label{theorem:bellman}
\end{theorem}

\begin{figure*}[t]
\centering
\begin{center}
  \includegraphics[width=0.9\linewidth]{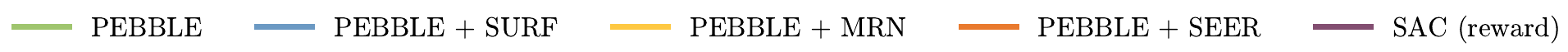}
  \vspace{-1.2em}
\end{center}
\begin{tabular}{cccc}
\hspace*{-1.1em}
\rotatebox{90}{\qquad \qquad Sokoban}
& \hspace*{-1.1em} \subfloat[Push-7x7-1 (feedback=300)]{\includegraphics[width=0.32\linewidth]{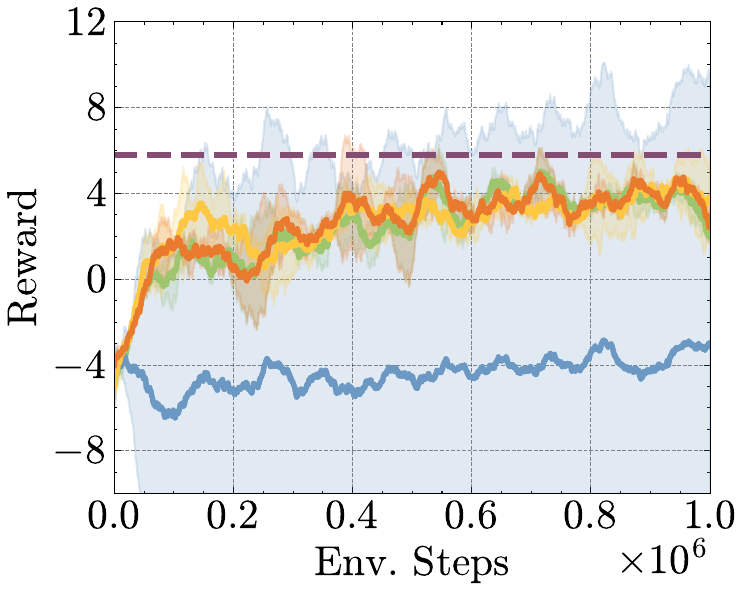}}
& \hspace*{-0.6em} \subfloat[Push-6x6-2 (feedback=1000)]{\includegraphics[width=0.32\linewidth]{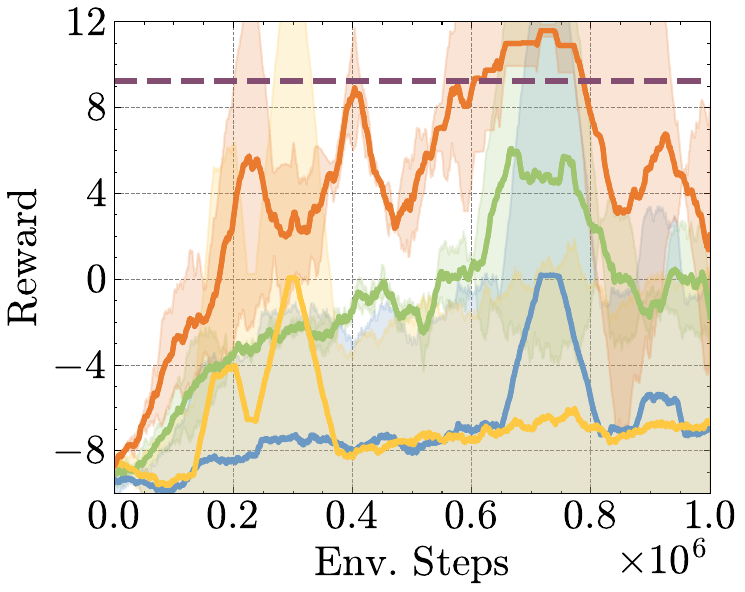}}
& \hspace*{-0.6em} \subfloat[Push-7x7-2 (feedback=1000)]{\includegraphics[width=0.32\linewidth]{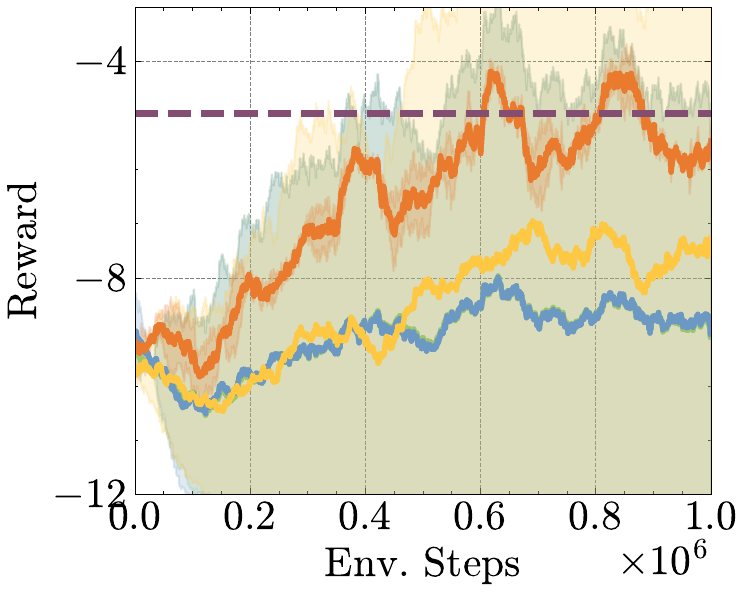}}
\vspace*{-0.5em} \\
\hspace*{-1.1em}
\rotatebox{90}{\qquad \qquad Craftenv}
& \hspace*{-1.1em} \subfloat[Strip-shaped (feedback=1000)]{\includegraphics[width=0.32\linewidth]{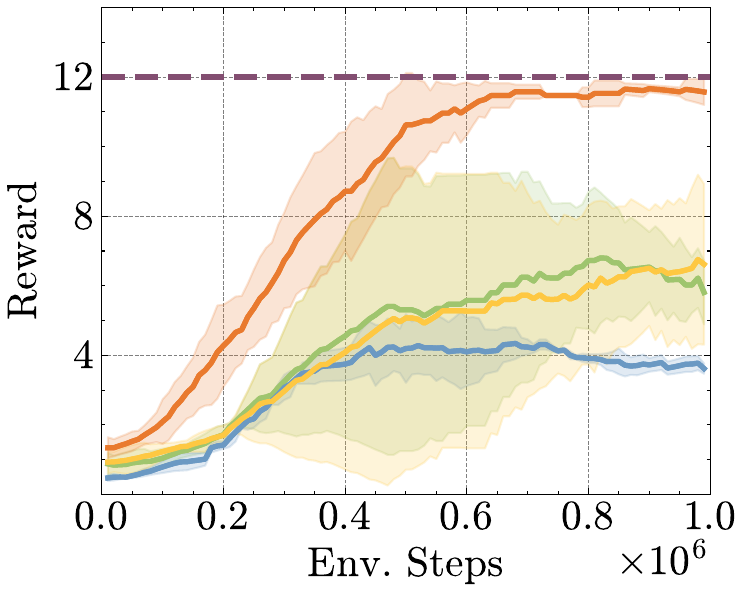}}
& \hspace*{-0.6em} \subfloat[Block-shaped (feedback=1000)]{\includegraphics[width=0.32\linewidth]{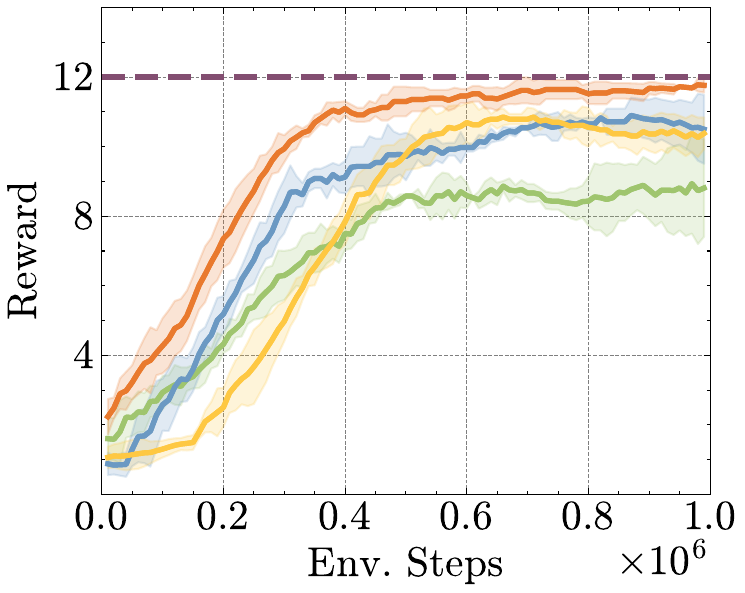}}
& \hspace*{-0.6em} \subfloat[Two-Story   (feedback=1000) ]{\includegraphics[width=0.32\linewidth]{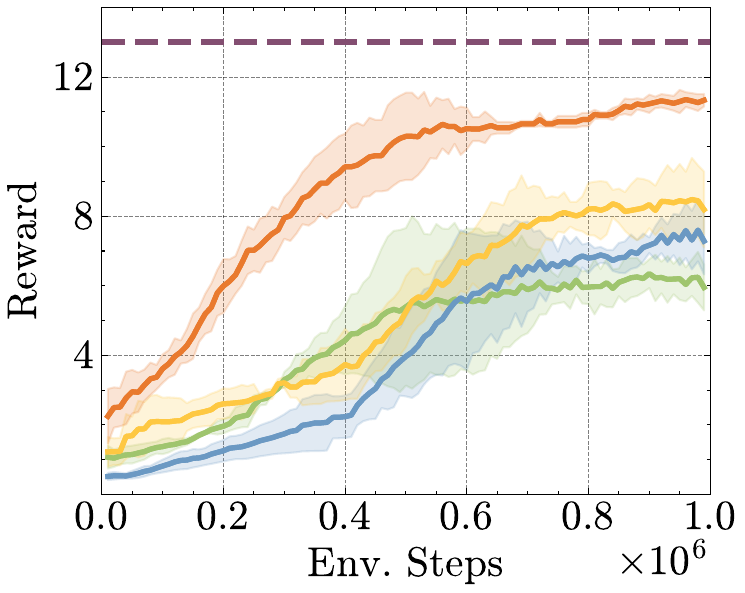}}
\end{tabular}
\caption{Training curves for all methods in discrete settings. The solid line indicates the mean values, while the shaded area denotes the standard deviations over five runs. The orange line is our method.}
\label{fig:main_result_discrete}
\vspace{-1.3em}
\end{figure*}

\section{Experiment}
In this section, we conduct evaluations across a variety of environments, including puzzle video games from Sokoban~\citep{SchraderSokoban2018}, robotic construction tasks from CraftEnv~\citep{ZhaoCraftEnv2023}, locomotion tasks from DeepMind Control Suite (DMControl)~\citep{tassa2018deepmind, tunyasuvunakool2020dm_control}, robotic manipulation challenges from Meta-world~\citep{yu2020meta}, and D4RL~\citep{fu2020d4rl}. Our aim is to investigate the following key questions:
\textbf{(1)} Does \ourmethod demonstrate better feedback efficiency in learning policies across different settings (discrete, continuous) in online environments?
\textbf{(2)} Is \ourmethod capable of training more effective policies in offline settings?
\textbf{(3)} Can \ourmethod achieve a more \textbf{accurate} Q function estimation?
\textbf{(4)} Are the individual components within \ourmethod \textbf{effective}?
The answers to questions (1) and (2) are provided in Section~\ref{sec:result_main}, supported by detailed empirical results. Questions (3) and (4) are explored in Section~\ref{sec:ablation} through ablation studies. Additional information about the tasks utilized in our experiments can be found in Appendix~\ref{appendix:task_specifications}.

\subsection{Experimental Settings} \label{subsec:setup}
\textbf{Baselines.} We compare our algorithm with several state-of-the-art methods, spanning both online, offline PbRL algorithms:
\textbf{Online Methods:}
(1) \textit{PEBBLE}~\citep{lee2021pebble}: integrates unsupervised pre-training with reward relabeling techniques during policy learning.
(2) \textit{SURF}~\citep{park2022surf}: employs temporal data augmentation and pseudo labels within a semi-supervised learning framework.
(3) \textit{MRN}~\citep{NEURIPS2022_8be9c134}: utilizes bi-level optimization methods in reward learning, which is the current state-of-the-art algorithm in online PbRL.
\textbf{Offline Methods:}
(4) \textit{PT}~\citep{kim2023preference}: utilizes transformer architecture to derive a non-Markovian reward and preference weighting function for offline PbRL.
(5) \textit{IPL}~\citep{abs-2305-15363}: directly optimizes the implicit rewards deduced from the learned Q-function, ensuring alignment with expert preferences.
\textbf{Reward-based Methods.} Explicit rewards are absent in PbRL settings. We utilize benchmarks that employ ground-truth rewards from the environment, specifically \textit{SAC}~\citep{haarnoja2018soft} for online scenarios, and \textit{IQL}~\citep{kostrikov2022offline} and \textit{TD3+BC}~\citep{NEURIPS2021_a8166da0} for offline scenarios.

\begin{figure*}[!ht]
\centering
\begin{center}
  \includegraphics[width=0.95\linewidth]{figures/results/main_header.pdf}
  \vspace{-1.2em}
\end{center}
\begin{tabular}{cccc}
\hspace*{-1.2em}
\rotatebox{90}{\qquad \qquad DMControl}
& \hspace*{-1.2em} \subfloat[Cheetah Run (feedback=100)]{\includegraphics[width=0.335\linewidth]{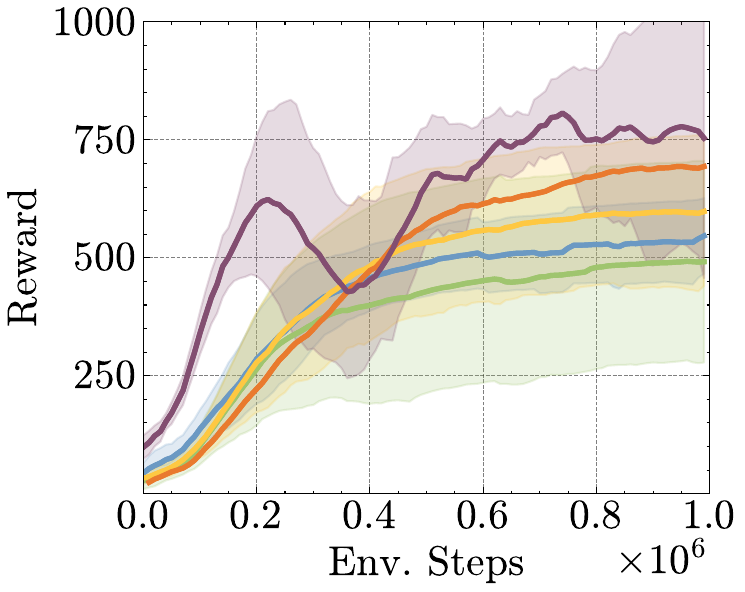}}
& \hspace*{-1.7em} \subfloat[Walker Walk (feedback=200)]{\includegraphics[width=0.335\linewidth]{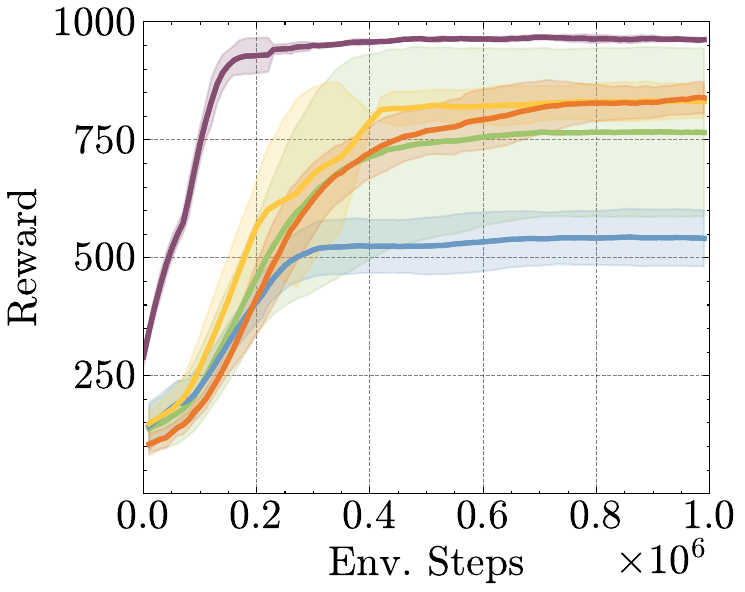}}
& \hspace*{-1.7em} \subfloat[Quadruped Walk (feedback=1000)]{\includegraphics[width=0.335\linewidth]{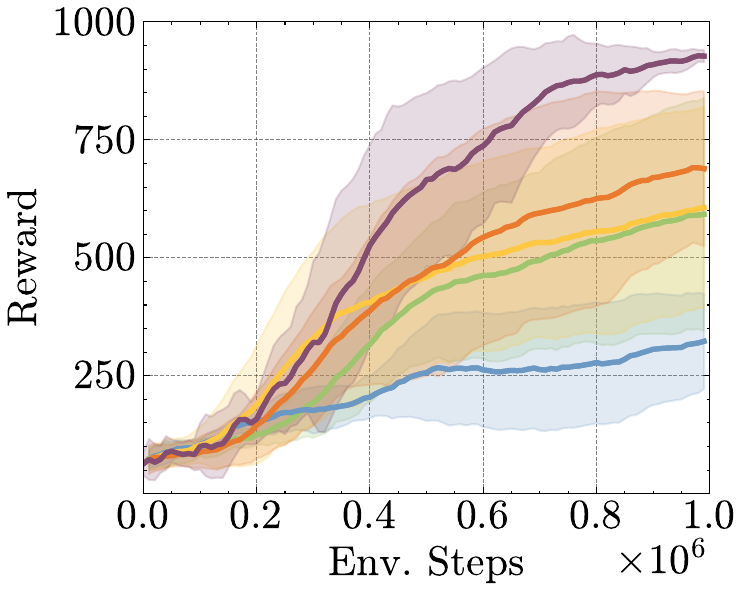}}
\vspace*{-0.5em} \\
\hspace*{-1.2em}
\rotatebox{90}{\qquad \qquad Meta-world}
& \hspace*{-1.2em} \subfloat[Button Press (feedback=100)]{\includegraphics[width=0.335\linewidth]{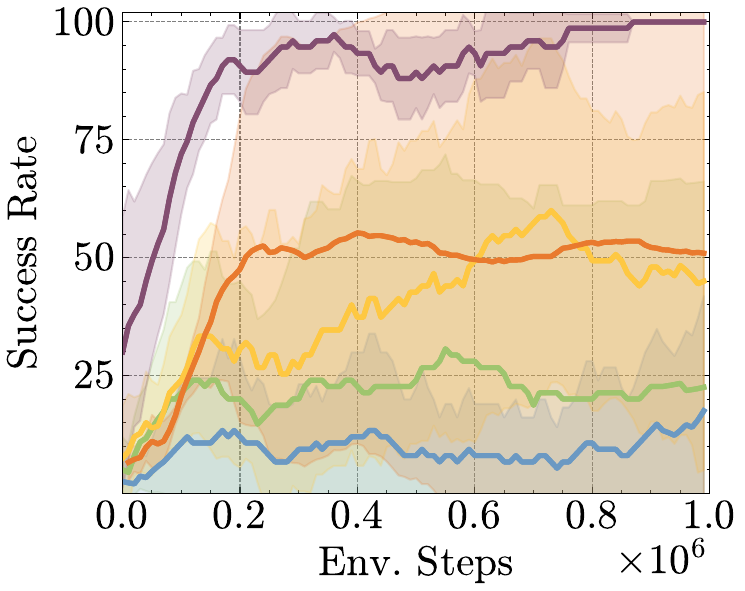}}
& \hspace*{-1.7em} \subfloat[Window Open (feedback=100)]{\includegraphics[width=0.335\linewidth]{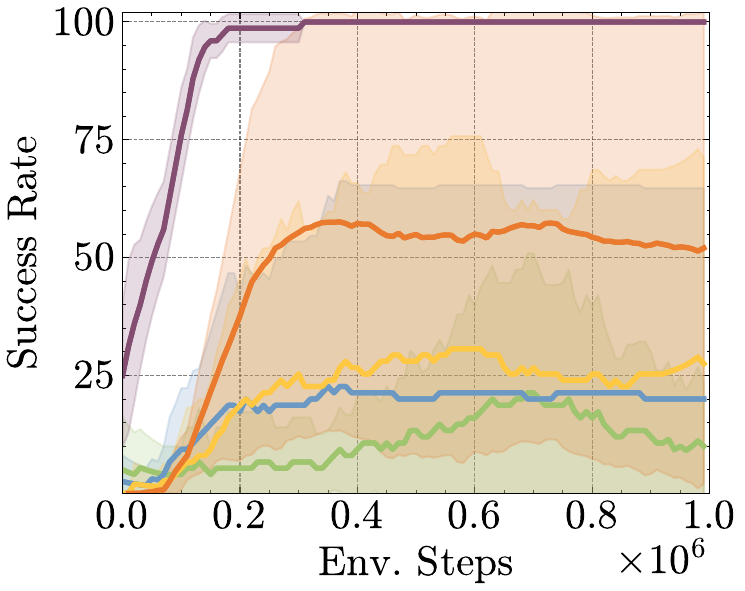}}
& \hspace*{-1.7em} \subfloat[Sweep Into (feedback=4000)]{\includegraphics[width=0.335\linewidth]{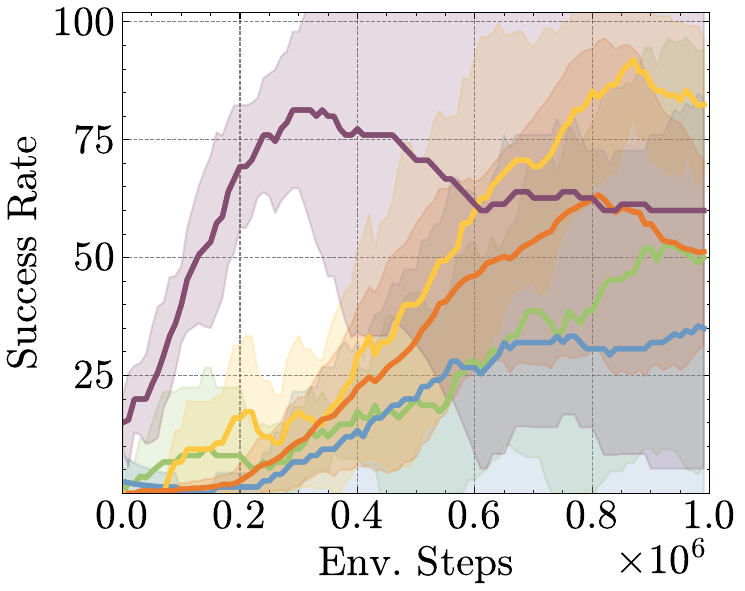}}
\end{tabular}
\caption{Evaluating curves of all methods on locomotion tasks and robotic manipulation tasks. The solid line presents the mean values, and the shaded area denotes the standard deviations over five runs. The orange line is our method.}
\label{fig:main_result_continuous}
% \vspace{-1.3em}
\end{figure*}

\noindent \textbf{Implementation details.}
In our experiments, we follow the basic settings employed in~\citep{lee2021pebble, park2022surf, NEURIPS2022_8be9c134}, which include unsupervised exploration techniques and an uncertainty-based trajectory sampling scheme (more details are available in Appendix~\ref{appendix:basic_pbrl}). Regarding the reward learning setting, all methods utilize an ensemble of three reward models, with their outputs confined to the range of $[-1, 1]$ through a hyperbolic tangent function. In line with prior research, we adopt a consistent approach for performance evaluation using a scripted teacher. This teacher supplies preference labels between pairs of trajectory segments based on the environment's inherent reward function. These preferences faithfully represent the environment's actual rewards, thus facilitating a quantitative comparison of the algorithms by assessing their true returns. It is important to emphasize that, within the PbRL framework, the agent does not have direct access to these rewards. As for offline settings, we use real-human preference data from \citet{kim2023preference}. And the number of preference labels required in each task is presented in Appendix~\ref{appendix:num_feedback}.

To ensure equitable comparisons, all methods are trained using the same network architecture and shared hyperparameters, except for method-specific elements. For baseline implementations, we use the publicly available code for PEBBLE~\footnote{\url{https://github.com/pokaxpoka/B_Pref}}, SURF~\footnote{\url{https://github.com/alinlab/SURF}}, MRN~\footnote{\url{https://github.com/RyanLiu112/MRN}}, and IPL~\footnote{\url{https://github.com/jhejna/inverse-preference-learning}}. More comprehensive details about the implementation of our approach and the baselines are provided in Appendix~\ref{appendix:imp_details}.

\subsection{Results}\label{sec:result_main}
\noindent \textbf{Discrete Settings.}
A detailed introduction and visualization of the six puzzle-solving tasks from Sokoban and the three flexible environments from CraftEnv are presented in Appendix~\ref{appendix:task_specifications}. We select these tasks for our experiments, covering a range of complexities. Figure~\ref{fig:main_result_discrete} depicts the learning curves of the average episode return for \ourmethod and baselines on discrete tasks. In each task, SAC utilizes the ground-truth reward, showcasing the best results as the performance upper bound. As observed in Figure~\ref{fig:main_result_discrete}, \ourmethod quickly achieves remarkable performance early in the training process across various tasks. Interestingly, in several tasks, \ourmethod nearly matches SAC's benchmark performance with only a limited number of human preference labels, indicating exceptional feedback efficiency. With a limited number of preference labels available, some baselines are significantly influenced by randomness in certain tasks, and their training curves show a downward trend in more challenging tasks. Notably, \ourmethod achieves performance comparable to PEBBLE, but with substantially fewer samples. For example, in the Strip-shaped building task, \ourmethod surpasses PEBBLE's average performance using only 30\% of the total samples. These results suggest that \ourmethod markedly reduces the feedback required to effectively tackle complex tasks, making it a highly efficient approach in PbRL.

\begin{table*}[t]
\centering
\myscriptsize % 更改字体大小
\caption{
Averaged normalized scores of all baselines on AntMaze, Gym-Mujoco locomotion tasks, and success rate on Robosuite manipulation tasks. All agents training use the same real human preferences dataset from~\citet{kim2023preference}. We train the \ourmethod and IPL, and report the average and standard deviation averaged over 15 runs. The term `reward' refers to the use of ground-truth rewards.}
\vspace{-0.5em}
\label{table:offline}
\renewcommand{\arraystretch}{1.2}%
\begin{center}
\aboverulesep = 0pt
\belowrulesep = 0pt
\resizebox{\textwidth}{!}{%
\begin{tabular}{l |cc|ccccc}
    \toprule
Dataset & IQL (reward) & TD3+BC (reward) & MR & LSTM & PT  & IPL  & \ourmethod(ours) \\
    \midrule
antmaze-medium-play-v2    & 73.88 \stdv{4.49}  & 0.25 \stdv{0.43} & 31.13 \stdv{16.96} & 62.88 \stdv{5.99}  & 70.13 \stdv{3.76} & 30.19 \stdv{4.97}  & 69.0  \stdv{14.97} \\
antmaze-medium-diverse-v2 & 68.13 \stdv{10.15} & 0.25 \stdv{0.43} & 19.38 \stdv{9.24}  & 20.13 \stdv{17.12} & 65.25 \stdv{3.59} & 24.21 \stdv{5.12}  & 67.0  \stdv{17.44} \\
antmaze-large-play-v2     & 48.75 \stdv{4.35}  & 0.0 \stdv{0.0}   & 24.25 \stdv{14.03} & 14.13 \stdv{3.60}  & 42.38 \stdv{9.98} & 12.46 \stdv{7.2}   & 50.67 \stdv{10.2}  \\
antmaze-large-diverse-v2  & 44.38 \stdv{4.47}  & 0.0 \stdv{0.0}   & 5.88  \stdv{6.94}  & 0.00  \stdv{0.00}  & 19.63 \stdv{3.70} & 0.0   \stdv{0.0}   & 48.0  \stdv{16.0}  \\
    \midrule
antmaze-v2 total & 58.79 & 0.13 & 20.16 & 24.29 & 49.35  & 16.72  & 58.67 \\
    \midrule
hopper-medium-replay-v2   & 83.06 \stdv{15.80} & 64.42 \stdv{21.52}& 11.56 \stdv{30.27}& 57.88 \stdv{40.63} & 84.54 \stdv{4.07}  & 73.57 \stdv{6.7}  & 85.29 \stdv{5.11} \\
hopper-medium-expert-v2   & 73.55 \stdv{41.47} & 101.17\stdv{9.07} & 57.75 \stdv{23.70}& 38.63 \stdv{35.58} & 68.96 \stdv{33.86} & 74.52 \stdv{0.1}  & 96.75 \stdv{4.31} \\
walker2d-medium-replay-v2 & 73.11 \stdv{8.07}  & 85.62 \stdv{4.01} & 72.07 \stdv{1.96} & 77.00\stdv{3.03}   & 71.27 \stdv{10.30} & 59.92 \stdv{5.1}  & 73.91 \stdv{1.33} \\
walker2d-medium-expert-v2 & 107.75\stdv{2.02}  & 110.03\stdv{0.36} & 108.32\stdv{3.87} & 110.39 \stdv{0.93} & 110.13 \stdv{0.21} & 108.51\stdv{0.6}  & 110.12\stdv{0.24} \\
    \midrule
locomotion-v2 total & 84.37 & 90.31 & 62.43 & \textcolor{black}{70.98} & 83.72  & 79.13  & 91.52 \\
    \midrule
lift-ph  &  96.75 \stdv{1.83} & -              & 84.75 \stdv{6.23} & 91.50 \stdv{5.42} & 91.75 \stdv{5.90}  & 97.6 \stdv{2.9}  & 98.0 \stdv{4.0}  \\
lift-mh  & 86.75 \stdv{2.82}  & -              & 91.00 \stdv{4.00} & 90.75 \stdv{5.75} & 86.75 \stdv{5.95}  & 87.2 \stdv{5.3}  & 93.0 \stdv{8.94} \\
can-ph   &  74.50 \stdv{6.82} & -              & 68.00 \stdv{9.13} & 62.00 \stdv{10.90}& 69.67 \stdv{5.89}  & 74.8 \stdv{2.4}  & 70.0 \stdv{13.56}\\
can-mh   &  56.25 \stdv{8.78} & -              & 47.50 \stdv{3.51} & 30.50 \stdv{8.73} & 50.50 \stdv{6.48}  & 57.6 \stdv{5.0}  & 47.0 \stdv{9.8}  \\
    \midrule
robosuite total & 78.56 & - & 72.81 & 68.69 & 74.66  & 79.3  & 77.0 \\
    \bottomrule
    \end{tabular}
    } % end of resize 
    \end{center}
\end{table*}

\noindent \textbf{Continuous Settings.}
We also evaluated \ourmethod's performance in continuous settings, focusing on three locomotion tasks from DMControl and three robotic simulated manipulation tasks from Meta-world. As depicted in Figure~\ref{fig:main_result_continuous}, \ourmethod surpasses the baselines in most tasks. Notably, in the Cheetah Run task, \ourmethod achieves 96\% of the best performance using only 100 preference labels, showcasing remarkable feedback efficiency. A comparison between the orange (\ourmethod) and green (PEBBLE) curves in the figures clearly indicates that \ourmethod improves upon PEBBLE's performance. These outcomes, as illustrated in Figure~\ref{fig:main_result_continuous}, demonstrate that \ourmethod enhances the feedback efficiency of PbRL methods across diverse and complex tasks. It is noteworthy that \ourmethod effectively utilizes the conservative estimate $\widehat{Q}$ function to regularize policy, resulting in improvements, especially with a limited number of preference labels.

\noindent \textbf{Offline Settings.} We benchmark \ourmethod against several offline PbRL algorithms on the D4RL~\citep{fu2020d4rl} and Robosuite~\citep{zhu2020robosuite} robotics datasets, utilizing real-human preference data from \citet{kim2023preference}. As before, we aim to avoid using out-of-samples actions. Therefore, we extract the policy using advantage weighted regression~\citep{1273496.1273590, NEURIPS2018_4aec1b34, nair2020awac}. In addition to comparing \ourmethod with two state-of-the-art offline PbRL algorithms, IPL and PT, we also benchmark it against two significant offline RL algorithms. As TD3 is our basic RL algorithm, we include results from TD3 enhanced with Behavior Cloning (TD3+BC)~\citep{NEURIPS2021_a8166da0}. The comparison also features different reward model architectures: MR (MLP-based), LSTM (LSTM-based), and PT (Transformer-based), corresponding to models proposed in ~\citet{NIPS2017_d5e2c0ad}, ~\citet{NEURIPS2022_b157cfde}, and ~\citet{kim2023preference} respectively. Our findings are presented in Table~\ref{table:offline}. \ourmethod consistently surpasses all baselines in nearly every task. Notably, it exhibits a clear advantage in challenging tasks like antmaze-large. Only our method almost matches the performance of IQL with the task reward and even surpasses IQL in some cases. These results demonstrate the effectiveness of our method across a variety of tasks in offline settings.

\subsection{Ablation Studies} \label{sec:ablation}
\noindent \textbf{Contribution of each technique.}
To assess the individual contributions of each technique in \ourmethod, we apply human label smoothing and policy regularization. Figure~(\ref{fig:pi_reg}) shows a comparative analysis of \ourmethod's performance with and without policy regularization on tasks such as Cheetah Run and Quadruped Walk. This visual representation clearly indicates that incorporating policy regularization enables agents to enhance their performance. Figure~(\ref{fig:label_smooth}) displays the impact of varying $\lambda$ values on \ourmethod's performance, underscoring the effectiveness of label smoothing in improving performance. However, we also note that while label smoothing is beneficial, an overly large $\lambda$ value can be detrimental. It is observed that in more complex tasks, a smaller $\lambda$ is preferable, whereas in less challenging tasks, a larger $\lambda$ value tends to yield better results.

\vspace{0.5em}
\noindent \textbf{Accuracy of value estimation.} 
We assess the accuracy of value estimation in \ourmethod by examining the value estimate trajectory during the learning process on Cheetah Run. Figure~(\ref{fig:value_estimation}) illustrates this by charting the average value estimate across 10000 states and contrasting it with an estimate of the true value. The true value is calculated based on the average discounted return obtained by following the current policy. Notably, a clear overestimation bias is evident in the learning procedure. When constrained by the policy regularizer, \ourmethod shows a significant reduction in overestimation bias, leading to a more accurate Q function. This improvement in Q estimation benefits the learning cycle of PbRL, making it a more effective approach.

\begin{figure*}[t]
% \vspace{-1em}
\centering
\begin{tabular}{ccc}
% \hspace*{-0.8em}
 \hspace*{-1.1em} \subfloat[Policy Regularization]{\includegraphics[width=0.33\linewidth]{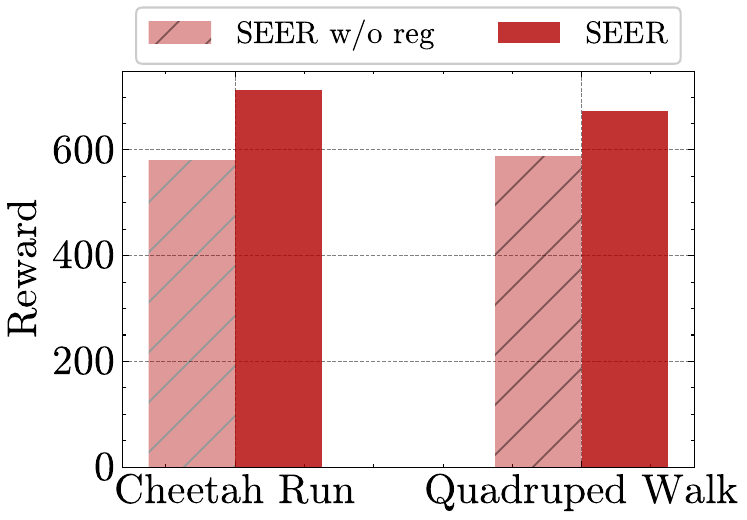}\label{fig:pi_reg}}
& \hspace*{-1.1em} \subfloat[Label Smoothing]{\includegraphics[width=0.33\linewidth]{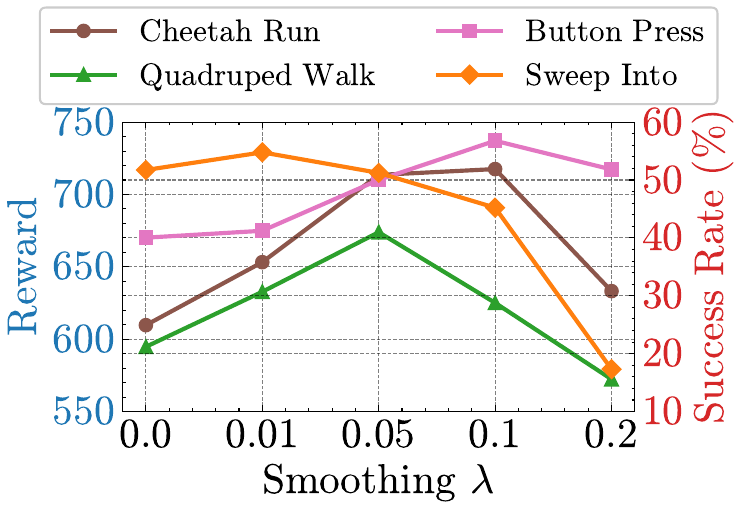}\label{fig:label_smooth}}
& \hspace*{-1.1em}  \subfloat[Value Estimation Accuracy]{\includegraphics[width=0.33\linewidth]{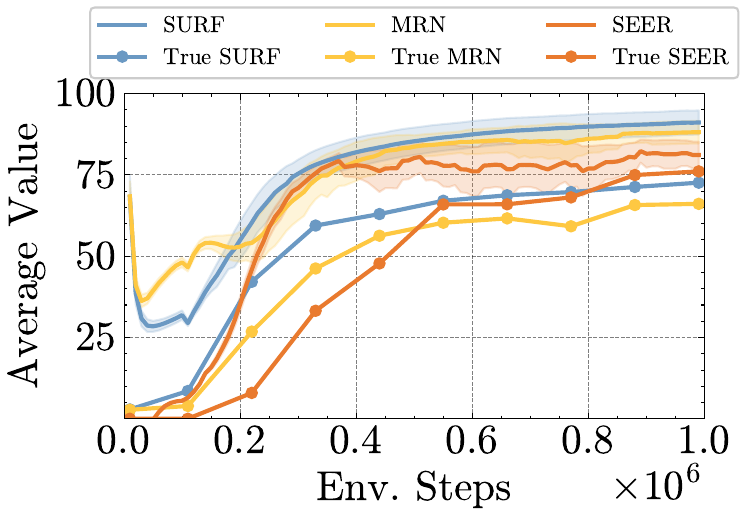}\label{fig:value_estimation}}
\end{tabular}
\caption{Ablation studies evaluating (a) the effectiveness of policy regularization, (b) the impact of the parameter $\lambda$  in label smoothing, and (c) the overestimation bias in value estimates.}
\label{fig:abla_weight}
\end{figure*}

\section{Related Work}
\noindent \textbf{Preference-based Reinforcement Learning.}
% \vspace{-0.5em}
Preference-based RL (PbRL) is a novel way to learn agents from human feedback without reward engineering. Instead, a human provides preferences between the agent’s behaviors, and the agent uses this feedback to perform the task. \citet{NIPS2017_d5e2c0ad} formulates a basic framework, and \citet{NEURIPS2018_8cbe9ce2} utilizes imitation learning as the warm-start strategy to speed up PbRL. To further improve feedback efficiency, \textbf{PEBBLE}~\citep{lee2021pebble} replaces PPO~\citep{SchulmanWDRK17} with SAC~\citep{haarnoja2018soft} to achieve data efficiency and combines unsupervised pre-training with reward relabeling technique during learning. Building upon this foundation, ~\citet{park2022surf} introduces \textbf{SURF}, a semi-supervised reward learning framework that improves reward learning via pseudo-labels and temporal cropping augmentation. \textbf{MRN}~\citep{NEURIPS2022_8be9c134} incorporates bi-level optimization for improving the quality of the Q function. Besides, some research has various considerations, such as skill extraction ~\citep{pmlr-v164-wang22g}, intrinsic reward ~\citep{liang2022reward}, meta-learning ~\citep{iii2022fewshot}, and these methods have improved the efficiency to a certain extent. In addition to focusing on PbRL in online settings, a significant portion of research also concentrates on PbRL in offline settings. \textbf{PT}~\citep{kim2023preference} leverages a transformer architecture to learn a non-Markovian reward and preference weighting function. \textbf{IPL}~\citep{abs-2305-15363} directly optimizes the implicit rewards deduced from the learned Q-function, ensuring alignment with expert preferences. ~\citet{pmlr-v202-kang23b} models offline trajectories and preferences in a one-step process without reward learning. 
When dealing with language models, PbRL naturally facilitates the emergence of reinforcement learning from human feedback (RLHF)~\citep{NEURIPS2022_b1efde53}. Before that, \citet{NEURIPS2020_1f89885d,abs-2109-10862} have fine-tuned a summarizing policy following the PbRL paradigm. Our approach is orthogonal to previous approaches, that we use conservative estimate $\widehat{Q}$ to regularize the neural Q-function to mitigate overestimation and smooth the human label to prevent overfitting of the reward model.

\section{Conclusion}
In this work, we present \ourmethod, a novel PbRL algorithm that notably enhances feedback efficiency. By integrating policy regularization and label smoothing, \ourmethod not only surpasses previous methods but also significantly improves feedback efficiency across various complex tasks in both online and offline settings. A key strength of our method is its remarkable performance with a limited number of preference labels. Our empirical results and analyses indicate that the enhanced feedback efficiency in \ourmethod primarily arises from two factors: \textbf{(1)} Our approach mitigates overestimation bias, contributing to a more precise Q-function estimation. \textbf{(2)} Label smoothing effectively reduces the reward model's tendency to overfit. We hope our method can provide inspiration for future work and encourage preference-based reinforcement learning to be better extended to practical applications. \ourmethod's success in addressing these critical aspects of PbRL demonstrates its potential to influence future research and expand the practical applicability of preference-based reinforcement learning in various domains.

\textbf{Limitations.} 
First, the preference learning in our algorithm is based on trajectories, which means it cannot accurately discriminate between good and bad actions within one trajectory. Additionally, we do not discuss high-dimensional and multi-modal inputs, such as images and natural language. While these are not the focus of this work, we consider this as an interesting future direction.

%%%%%%%%%%%%%%%%%%%%%%%%%%%%%%%%%%%%%%%%%%%%%%%%%%%%%%%%%%%%%%%%%%%%%%%%%%%%%%%
%%%%%%%%%%%%%%%%%%%%%%%%%%%%%%%%%%%%%%%%%%%%%%%%%%%%%%%%%%%%%%%%%%%%%%%%%%%%%%%
% References
% \newpage
\bibliography{neurips_2024}

\begin{thebibliography}{58}
\providecommand{\natexlab}[1]{#1}
\providecommand{\url}[1]{\texttt{#1}}
\expandafter\ifx\csname urlstyle\endcsname\relax
  \providecommand{\doi}[1]{doi: #1}\else
  \providecommand{\doi}{doi: \begingroup \urlstyle{rm}\Url}\fi

\bibitem[Berner et~al.(2019)Berner, Brockman, Chan, Cheung, Debiak, Dennison, Farhi, Fischer, Hashme, Hesse, J{\'{o}}zefowicz, Gray, Olsson, Pachocki, Petrov, de~Oliveira~Pinto, Raiman, Salimans, Schlatter, Schneider, Sidor, Sutskever, Tang, Wolski, and Zhang]{berner2019dota}
Christopher Berner, Greg Brockman, Brooke Chan, Vicki Cheung, Przemyslaw Debiak, Christy Dennison, David Farhi, Quirin Fischer, Shariq Hashme, Christopher Hesse, Rafal J{\'{o}}zefowicz, Scott Gray, Catherine Olsson, Jakub Pachocki, Michael Petrov, Henrique~Pond{\'{e}} de~Oliveira~Pinto, Jonathan Raiman, Tim Salimans, Jeremy Schlatter, Jonas Schneider, Szymon Sidor, Ilya Sutskever, Jie Tang, Filip Wolski, and Susan Zhang.
\newblock Dota 2 with large scale deep reinforcement learning.
\newblock \emph{CoRR}, abs/1912.06680, 2019.

\bibitem[Bradley and Terry(1952)]{pref-model-orig}
Ralph~Allan Bradley and Milton~E. Terry.
\newblock Rank analysis of incomplete block designs: I. the method of paired comparisons.
\newblock \emph{Biometrika}, 39\penalty0 (3/4):\penalty0 324--345, 1952.

\bibitem[Brohan et~al.(2023)Brohan, Brown, Carbajal, Chebotar, Chen, Choromanski, Ding, Driess, Dubey, Finn, et~al.]{brohan2023rt}
Anthony Brohan, Noah Brown, Justice Carbajal, Yevgen Chebotar, Xi~Chen, Krzysztof Choromanski, Tianli Ding, Danny Driess, Avinava Dubey, Chelsea Finn, et~al.
\newblock Rt-2: Vision-language-action models transfer web knowledge to robotic control.
\newblock \emph{arXiv preprint arXiv:2307.15818}, 2023.

\bibitem[Christiano et~al.(2017)Christiano, Leike, Brown, Martic, Legg, and Amodei]{NIPS2017_d5e2c0ad}
Paul~F Christiano, Jan Leike, Tom Brown, Miljan Martic, Shane Legg, and Dario Amodei.
\newblock Deep reinforcement learning from human preferences.
\newblock In \emph{Advances in Neural Information Processing Systems (NeurIPS)}, volume~30, 2017.

\bibitem[Early et~al.(2022)Early, Bewley, Evers, and Ramchurn]{NEURIPS2022_b157cfde}
Joseph Early, Tom Bewley, Christine Evers, and Sarvapali Ramchurn.
\newblock Non-markovian reward modelling from trajectory labels via interpretable multiple instance learning.
\newblock In \emph{Advances in Neural Information Processing Systems (NeurIPS)}, volume~35, pages 27652--27663, 2022.

\bibitem[Fu et~al.(2020)Fu, Kumar, Nachum, Tucker, and Levine]{fu2020d4rl}
Justin Fu, Aviral Kumar, Ofir Nachum, George Tucker, and Sergey Levine.
\newblock D4rl: Datasets for deep data-driven reinforcement learning.
\newblock \emph{arXiv preprint arXiv:2004.07219}, 2020.

\bibitem[Fujimoto and Gu(2021)]{NEURIPS2021_a8166da0}
Scott Fujimoto and Shixiang~(Shane) Gu.
\newblock A minimalist approach to offline reinforcement learning.
\newblock In \emph{Advances in Neural Information Processing Systems (NeurIPS)}, volume~34, pages 20132--20145, 2021.

\bibitem[Fujimoto et~al.(2018)Fujimoto, van Hoof, and Meger]{pmlr-v80-fujimoto18a}
Scott Fujimoto, Herke van Hoof, and David Meger.
\newblock Addressing function approximation error in actor-critic methods.
\newblock In \emph{International Conference on Machine Learning (ICML)}, volume~80, pages 1587--1596, 10--15 Jul 2018.

\bibitem[Fujimoto et~al.(2019)Fujimoto, Meger, and Precup]{fujimoto2019off}
Scott Fujimoto, David Meger, and Doina Precup.
\newblock Off-policy deep reinforcement learning without exploration.
\newblock In \emph{International Conference on Machine Learning}, pages 2052--2062. PMLR, 2019.

\bibitem[Garg et~al.(2023)Garg, Hejna, Geist, and Ermon]{garg2023extreme}
Divyansh Garg, Joey Hejna, Matthieu Geist, and Stefano Ermon.
\newblock Extreme q-learning: Maxent {RL} without entropy.
\newblock In \emph{International Conference on Learning Representations (ICLR)}, 2023.

\bibitem[Gu et~al.(2017)Gu, Holly, Lillicrap, and Levine]{7989385}
Shixiang Gu, Ethan Holly, Timothy Lillicrap, and Sergey Levine.
\newblock Deep reinforcement learning for robotic manipulation with asynchronous off-policy updates.
\newblock In \emph{IEEE International Conference on Robotics and Automation (ICRA)}, pages 3389--3396, 2017.

\bibitem[Haarnoja et~al.(2018)Haarnoja, Zhou, Abbeel, and Levine]{haarnoja2018soft}
Tuomas Haarnoja, Aurick Zhou, Pieter Abbeel, and Sergey Levine.
\newblock Soft actor-critic: Off-policy maximum entropy deep reinforcement learning with a stochastic actor.
\newblock In \emph{International Conference on Machine Learning (ICML)}, volume~80, pages 1861--1870, 2018.

\bibitem[Hejna and Sadigh(2023)]{abs-2305-15363}
Joey Hejna and Dorsa Sadigh.
\newblock Inverse preference learning: Preference-based {RL} without a reward function.
\newblock 2023.

\bibitem[Hejna~III and Sadigh(2023)]{iii2022fewshot}
Donald~Joseph Hejna~III and Dorsa Sadigh.
\newblock Few-shot preference learning for human-in-the-loop {RL}.
\newblock In \emph{International Conference on Robot Learning (CoRL)}, pages 2014--2025. PMLR, 2023.

\bibitem[Hong et~al.(2022)Hong, Chen, Lin, Pajarinen, and Agrawal]{hong2022topological}
Zhang-Wei Hong, Tao Chen, Yen-Chen Lin, Joni Pajarinen, and Pulkit Agrawal.
\newblock Topological experience replay.
\newblock In \emph{International Conference on Learning Representations (ICLR)}, 2022.

\bibitem[Ibarz et~al.(2018)Ibarz, Leike, Pohlen, Irving, Legg, and Amodei]{NEURIPS2018_8cbe9ce2}
Borja Ibarz, Jan Leike, Tobias Pohlen, Geoffrey Irving, Shane Legg, and Dario Amodei.
\newblock Reward learning from human preferences and demonstrations in atari.
\newblock In \emph{Advances in Neural Information Processing Systems (NeurIPS)}, volume~31, 2018.

\bibitem[Jaakkola et~al.(1993)Jaakkola, Jordan, and Singh]{jaakkola1993convergence}
Tommi Jaakkola, Michael Jordan, and Satinder Singh.
\newblock Convergence of stochastic iterative dynamic programming algorithms.
\newblock \emph{Advances in neural information processing systems}, 6, 1993.

\bibitem[Kang et~al.(2023)Kang, Shi, Liu, He, and Wang]{pmlr-v202-kang23b}
Yachen Kang, Diyuan Shi, Jinxin Liu, Li~He, and Donglin Wang.
\newblock Beyond reward: Offline preference-guided policy optimization.
\newblock In \emph{International Conference on Machine Learning (ICML)}, volume 202, pages 15753--15768, 23--29 Jul 2023.

\bibitem[Kim et~al.(2023)Kim, Park, Shin, Lee, Abbeel, and Lee]{kim2023preference}
Changyeon Kim, Jongjin Park, Jinwoo Shin, Honglak Lee, Pieter Abbeel, and Kimin Lee.
\newblock Preference transformer: Modeling human preferences using transformers for {RL}.
\newblock In \emph{International Conference on Learning Representations (ICLR)}, 2023.

\bibitem[Kostrikov et~al.(2022)Kostrikov, Nair, and Levine]{kostrikov2022offline}
Ilya Kostrikov, Ashvin Nair, and Sergey Levine.
\newblock Offline reinforcement learning with implicit q-learning.
\newblock In \emph{International Conference on Learning Representations (ICLR)}, 2022.

\bibitem[Kumar et~al.(2019)Kumar, Fu, Soh, Tucker, and Levine]{kumar2019stabilizing}
Aviral Kumar, Justin Fu, Matthew Soh, George Tucker, and Sergey Levine.
\newblock Stabilizing off-policy q-learning via bootstrapping error reduction.
\newblock \emph{Advances in Neural Information Processing Systems}, 32, 2019.

\bibitem[Kumar et~al.(2020)Kumar, Zhou, Tucker, and Levine]{NEURIPS2020_0d2b2061}
Aviral Kumar, Aurick Zhou, George Tucker, and Sergey Levine.
\newblock Conservative q-learning for offline reinforcement learning.
\newblock In \emph{Advances in Neural Information Processing Systems (NeurIPS)}, volume~33, pages 1179--1191, 2020.

\bibitem[Lee et~al.(2021)Lee, Smith, and Abbeel]{lee2021pebble}
Kimin Lee, Laura~M Smith, and Pieter Abbeel.
\newblock Pebble: Feedback-efficient interactive reinforcement learning via relabeling experience and unsupervised pre-training.
\newblock In \emph{International Conference on Machine Learning (ICML)}, volume 139, pages 6152--6163, 2021.

\bibitem[Lee et~al.(2019)Lee, Sungik, and Chung]{NEURIPS2019_e6d8545d}
Su~Young Lee, Choi Sungik, and Sae-Young Chung.
\newblock Sample-efficient deep reinforcement learning via episodic backward update.
\newblock In \emph{Advances in Neural Information Processing Systems (NeurIPS)}, volume~32, 2019.

\bibitem[Levine et~al.(2020)Levine, Kumar, Tucker, and Fu]{levine2020offline}
Sergey Levine, Aviral Kumar, George Tucker, and Justin Fu.
\newblock Offline reinforcement learning: Tutorial, review, and perspectives on open problems.
\newblock \emph{arXiv preprint arXiv:2005.01643}, 2020.

\bibitem[Liang et~al.(2022)Liang, Shu, Lee, and Abbeel]{liang2022reward}
Xinran Liang, Katherine Shu, Kimin Lee, and Pieter Abbeel.
\newblock Reward uncertainty for exploration in preference-based reinforcement learning.
\newblock In \emph{International Conference on Learning Representations (ICLR)}, 2022.

\bibitem[Lillicrap et~al.(2016)Lillicrap, Hunt, Pritzel, Heess, Erez, Tassa, Silver, and Wierstra]{lillicrap2015continuous}
Timothy~P. Lillicrap, Jonathan~J. Hunt, Alexander Pritzel, Nicolas Heess, Tom Erez, Yuval Tassa, David Silver, and Daan Wierstra.
\newblock Continuous control with deep reinforcement learning.
\newblock In \emph{International Conference on Learning Representations (ICLR)}, 2016.

\bibitem[Liu et~al.(2022)Liu, Bai, Du, and Yang]{NEURIPS2022_8be9c134}
Runze Liu, Fengshuo Bai, Yali Du, and Yaodong Yang.
\newblock Meta-reward-net: Implicitly differentiable reward learning for preference-based reinforcement learning.
\newblock In \emph{Advances in Neural Information Processing Systems (NeurIPS)}, volume~35, pages 22270--22284, 2022.

\bibitem[Melo(2001)]{melo2001convergence}
Francisco~S Melo.
\newblock Convergence of q-learning: A simple proof.
\newblock \emph{Institute Of Systems and Robotics, Tech. Rep}, pages 1--4, 2001.

\bibitem[Misra et~al.(2003)Misra, Singh, and Hnizdo]{10737616}
Neeraj Misra, Harshinder Singh, and Vladimir Hnizdo.
\newblock Nearest neighbor estimates of entropy.
\newblock \emph{American Journal of Mathematical and Management Sciences}, 23\penalty0 (3-4):\penalty0 301--321, 2003.

\bibitem[Mnih et~al.(2013{\natexlab{a}})Mnih, Kavukcuoglu, Silver, Graves, Antonoglou, Wierstra, and Riedmiller]{Dmnih2013playing}
Volodymyr Mnih, Koray Kavukcuoglu, David Silver, Alex Graves, Ioannis Antonoglou, Daan Wierstra, and Martin~A. Riedmiller.
\newblock Playing atari with deep reinforcement learning.
\newblock \emph{CoRR}, abs/1312.5602, 2013{\natexlab{a}}.

\bibitem[Mnih et~al.(2013{\natexlab{b}})Mnih, Kavukcuoglu, Silver, Graves, Antonoglou, Wierstra, and Riedmiller]{MnihKSGAWR13}
Volodymyr Mnih, Koray Kavukcuoglu, David Silver, Alex Graves, Ioannis Antonoglou, Daan Wierstra, and Martin~A. Riedmiller.
\newblock Playing atari with deep reinforcement learning.
\newblock \emph{CoRR}, abs/1312.5602, 2013{\natexlab{b}}.

\bibitem[Nair et~al.(2020)Nair, Gupta, Dalal, and Levine]{nair2020awac}
Ashvin Nair, Abhishek Gupta, Murtaza Dalal, and Sergey Levine.
\newblock Awac: Accelerating online reinforcement learning with offline datasets.
\newblock \emph{arXiv preprint arXiv:2006.09359}, 2020.

\bibitem[Ouyang et~al.(2022)Ouyang, Wu, Jiang, Almeida, Wainwright, Mishkin, Zhang, Agarwal, Slama, Ray, Schulman, Hilton, Kelton, Miller, Simens, Askell, Welinder, Christiano, Leike, and Lowe]{NEURIPS2022_b1efde53}
Long Ouyang, Jeffrey Wu, Xu~Jiang, Diogo Almeida, Carroll Wainwright, Pamela Mishkin, Chong Zhang, Sandhini Agarwal, Katarina Slama, Alex Ray, John Schulman, Jacob Hilton, Fraser Kelton, Luke Miller, Maddie Simens, Amanda Askell, Peter Welinder, Paul~F Christiano, Jan Leike, and Ryan Lowe.
\newblock Training language models to follow instructions with human feedback.
\newblock In \emph{Advances in Neural Information Processing Systems (NeurIPS)}, volume~35, pages 27730--27744, 2022.

\bibitem[Park et~al.(2022)Park, Seo, Shin, Lee, Abbeel, and Lee]{park2022surf}
Jongjin Park, Younggyo Seo, Jinwoo Shin, Honglak Lee, Pieter Abbeel, and Kimin Lee.
\newblock {SURF}: Semi-supervised reward learning with data augmentation for feedback-efficient preference-based reinforcement learning.
\newblock In \emph{International Conference on Learning Representations (ICLR)}, 2022.

\bibitem[Peters and Schaal(2007)]{1273496.1273590}
Jan Peters and Stefan Schaal.
\newblock Reinforcement learning by reward-weighted regression for operational space control.
\newblock In \emph{International Conference on Machine Learning (ICML)}, page 745–750, 2007.

\bibitem[Pham et~al.(2021)Pham, Dai, Xie, and Le]{pham2021meta_pseudo_labels}
Hieu Pham, Zihang Dai, Qizhe Xie, and Quoc~V. Le.
\newblock Meta pseudo labels.
\newblock In \emph{IEEE/CVF Conference on Computer Vision and Pattern Recognition (CVPR)}, pages 11557--11568, 2021.

\bibitem[Robbins and Monro(1951)]{robbins1951stochastic}
Herbert Robbins and Sutton Monro.
\newblock A stochastic approximation method.
\newblock \emph{The annals of mathematical statistics}, pages 400--407, 1951.

\bibitem[Rotinov(2019)]{abs-1910-08780}
Egor Rotinov.
\newblock Reverse experience replay.
\newblock 2019.

\bibitem[Schenck and Fox(2017)]{7989307}
Connor Schenck and Dieter Fox.
\newblock Visual closed-loop control for pouring liquids.
\newblock In \emph{International Conference on Robotics and Automation (ICRA)}, pages 2629--2636, 2017.

\bibitem[Schrader(2018)]{SchraderSokoban2018}
Max-Philipp~B. Schrader.
\newblock gym-sokoban.
\newblock \url{https://github.com/mpSchrader/gym-sokoban}, 2018.

\bibitem[Schulman et~al.(2017)Schulman, Wolski, Dhariwal, Radford, and Klimov]{SchulmanWDRK17}
John Schulman, Filip Wolski, Prafulla Dhariwal, Alec Radford, and Oleg Klimov.
\newblock Proximal policy optimization algorithms.
\newblock \emph{CoRR}, abs/1707.06347, 2017.

\bibitem[Skalse et~al.(2022)Skalse, Howe, Krasheninnikov, and Krueger]{NEURIPS2022_3d719fee}
Joar Skalse, Nikolaus Howe, Dmitrii Krasheninnikov, and David Krueger.
\newblock Defining and characterizing reward gaming.
\newblock In \emph{Advances in Neural Information Processing Systems (NeurIPS)}, volume~35, pages 9460--9471, 2022.

\bibitem[Stiennon et~al.(2020)Stiennon, Ouyang, Wu, Ziegler, Lowe, Voss, Radford, Amodei, and Christiano]{NEURIPS2020_1f89885d}
Nisan Stiennon, Long Ouyang, Jeffrey Wu, Daniel Ziegler, Ryan Lowe, Chelsea Voss, Alec Radford, Dario Amodei, and Paul~F Christiano.
\newblock Learning to summarize with human feedback.
\newblock In \emph{Advances in Neural Information Processing Systems (NeurIPS)}, volume~33, pages 3008--3021, 2020.

\bibitem[Szepesv{\'a}ri(2010)]{szepesvari2010algorithms}
Csaba Szepesv{\'a}ri.
\newblock Algorithms for reinforcement learning.
\newblock \emph{Synthesis lectures on artificial intelligence and machine learning}, 4\penalty0 (1):\penalty0 1--103, 2010.

\bibitem[Tassa et~al.(2018)Tassa, Doron, Muldal, Erez, Li, Casas, Budden, Abdolmaleki, Merel, Lefrancq, et~al.]{tassa2018deepmind}
Yuval Tassa, Yotam Doron, Alistair Muldal, Tom Erez, Yazhe Li, Diego de~Las Casas, David Budden, Abbas Abdolmaleki, Josh Merel, Andrew Lefrancq, et~al.
\newblock Deepmind control suite.
\newblock \emph{arXiv preprint arXiv:1801.00690}, 2018.

\bibitem[Tunyasuvunakool et~al.(2020)Tunyasuvunakool, Muldal, Doron, Liu, Bohez, Merel, Erez, Lillicrap, Heess, and Tassa]{tunyasuvunakool2020dm_control}
Saran Tunyasuvunakool, Alistair Muldal, Yotam Doron, Siqi Liu, Steven Bohez, Josh Merel, Tom Erez, Timothy Lillicrap, Nicolas Heess, and Yuval Tassa.
\newblock dm\_control: Software and tasks for continuous control.
\newblock \emph{Software Impacts}, 6:\penalty0 100022, 2020.

\bibitem[Victoria et~al.(2020)Victoria, Vladimir, Matthew, Tom, Ramana, Zac, Jan, and Shane]{dp_blog_2020}
Krakovna Victoria, Mikulik Vladimir, Rahtz Matthew, Everitt Tom, Kumar Ramana, Kenton Zac, Leike Jan, and Legg Shane.
\newblock Specification gaming: the flip side of ai ingenuity.
\newblock 2020.
\newblock URL \url{https://www.deepmind.com/blog/specification-gaming-the-flip-side-of-ai-ingenuity}.

\bibitem[Vinyals et~al.(2019)Vinyals, Babuschkin, Czarnecki, Mathieu, Dudzik, Chung, Choi, Powell, Ewalds, Georgiev, Oh, Horgan, Kroiss, Danihelka, Huang, Sifre, Cai, Agapiou, Jaderberg, Vezhnevets, Leblond, Pohlen, Dalibard, Budden, Sulsky, Molloy, Paine, G{\"{u}}l{\c{c}}ehre, Wang, Pfaff, Wu, Ring, Yogatama, W{\"{u}}nsch, McKinney, Smith, Schaul, Lillicrap, Kavukcuoglu, Hassabis, Apps, and Silver]{VinyalsBCMDCCPE19}
Oriol Vinyals, Igor Babuschkin, Wojciech~M. Czarnecki, Micha{\"{e}}l Mathieu, Andrew Dudzik, Junyoung Chung, David~H. Choi, Richard Powell, Timo Ewalds, Petko Georgiev, Junhyuk Oh, Dan Horgan, Manuel Kroiss, Ivo Danihelka, Aja Huang, Laurent Sifre, Trevor Cai, John~P. Agapiou, Max Jaderberg, Alexander~Sasha Vezhnevets, R{\'{e}}mi Leblond, Tobias Pohlen, Valentin Dalibard, David Budden, Yury Sulsky, James Molloy, Tom~Le Paine, {\c{C}}aglar G{\"{u}}l{\c{c}}ehre, Ziyu Wang, Tobias Pfaff, Yuhuai Wu, Roman Ring, Dani Yogatama, Dario W{\"{u}}nsch, Katrina McKinney, Oliver Smith, Tom Schaul, Timothy~P. Lillicrap, Koray Kavukcuoglu, Demis Hassabis, Chris Apps, and David Silver.
\newblock Grandmaster level in starcraft {II} using multi-agent reinforcement learning.
\newblock \emph{Nature}, 575\penalty0 (7782):\penalty0 350--354, 2019.

\bibitem[Wang et~al.(2018)Wang, Xiong, Han, sun, Liu, and Zhang]{NEURIPS2018_4aec1b34}
Qing Wang, Jiechao Xiong, Lei Han, peng sun, Han Liu, and Tong Zhang.
\newblock Exponentially weighted imitation learning for batched historical data.
\newblock In \emph{Advances in Neural Information Processing Systems (NeurIPS)}, volume~31, 2018.

\bibitem[Wang et~al.(2022)Wang, Lee, Hakhamaneshi, Abbeel, and Laskin]{pmlr-v164-wang22g}
Xiaofei Wang, Kimin Lee, Kourosh Hakhamaneshi, Pieter Abbeel, and Michael Laskin.
\newblock Skill preferences: Learning to extract and execute robotic skills from human feedback.
\newblock In \emph{Conference on Robot Learning (CoRL)}, volume 164, pages 1259--1268, 08--11 Nov 2022.

\bibitem[Wu et~al.(2021)Wu, Ouyang, Ziegler, Stiennon, Lowe, Leike, and Christiano]{abs-2109-10862}
Jeff Wu, Long Ouyang, Daniel~M. Ziegler, Nisan Stiennon, Ryan Lowe, Jan Leike, and Paul~F. Christiano.
\newblock Recursively summarizing books with human feedback.
\newblock \emph{CoRR}, abs/2109.10862, 2021.

\bibitem[Xu and Yu(2023)]{xu2023drl}
Yinda Xu and Lidong Yu.
\newblock Drl-based trajectory tracking for motion-related modules in autonomous driving.
\newblock \emph{arXiv preprint arXiv:2308.15991}, 2023.

\bibitem[Yahya et~al.(2017)Yahya, Li, Kalakrishnan, Chebotar, and Levine]{8202141}
Ali Yahya, Adrian Li, Mrinal Kalakrishnan, Yevgen Chebotar, and Sergey Levine.
\newblock Collective robot reinforcement learning with distributed asynchronous guided policy search.
\newblock In \emph{International Conference on Intelligent Robots and Systems (IROS)}, pages 79--86, 2017.

\bibitem[Yu et~al.(2020)Yu, Quillen, He, Julian, Hausman, Finn, and Levine]{yu2020meta}
Tianhe Yu, Deirdre Quillen, Zhanpeng He, Ryan Julian, Karol Hausman, Chelsea Finn, and Sergey Levine.
\newblock Meta-world: A benchmark and evaluation for multi-task and meta reinforcement learning.
\newblock In \emph{Conference on Robot Learning (CoRL)}, volume 100, pages 1094--1100. PMLR, 2020.

\bibitem[Zhao et~al.(2023)Zhao, Liu, Zhang, Li, Zhou, Li, and Han]{ZhaoCraftEnv2023}
Rui Zhao, Xu~Liu, Yizheng Zhang, Minghao Li, Cheng Zhou, Shuai Li, and Lei Han.
\newblock Craftenv: A flexible collective robotic construction environment for multi-agent reinforcement learning.
\newblock In \emph{International Conference on Autonomous Agents and Multiagent Systems (AAMAS)}, page 1164–1172, 2023.

\bibitem[Zhu et~al.(2020{\natexlab{a}})Zhu, Lin, Yang, and Zhang]{Zhu2020Episodic}
Guangxiang Zhu, Zichuan Lin, Guangwen Yang, and Chongjie Zhang.
\newblock Episodic reinforcement learning with associative memory.
\newblock In \emph{International Conference on Learning Representations (ICLR)}, 2020{\natexlab{a}}.

\bibitem[Zhu et~al.(2020{\natexlab{b}})Zhu, Wong, Mandlekar, and Mart{\'\i}n-Mart{\'\i}n]{zhu2020robosuite}
Yuke Zhu, Josiah Wong, Ajay Mandlekar, and Roberto Mart{\'\i}n-Mart{\'\i}n.
\newblock robosuite: A modular simulation framework and benchmark for robot learning.
\newblock \emph{arXiv preprint arXiv:2009.12293}, 2020{\natexlab{b}}.

\end{thebibliography}
\bibliographystyle{plainnat}

% Appendix 
% \appendix
\newpage
\appendix
\onecolumn
\setcounter{theorem}{0}

\section{The Full Procedure of \ourmethod}
The detailed procedures of our proposed method are outlined in Algorithm~\ref{pseudo_code}. 
Our method is based on the established framework of preference-based RL, PEBBLE~\cite{lee2021pebble}. 

In online settings for continuous action spaces, we utilize the constrained $\max$ operator for conservative estimation $\widehat{Q}$. Conversely, in discrete action spaces, a neural network $Q_\xi$ is employed for estimation. In offline settings, our approach initially focuses on training the reward model $\widehat{r}_\psi$, followed by policy learning.

\begin{algorithm}[]
\caption{\ourmethod (Online)}
\label{pseudo_code}
\begin{algorithmic}[1]
    \REQUIRE preference query frequency $K$, number of human's preference labels per session $M$
    \STATE Initialize parameters of $Q_\theta$, $\pi_\phi$, $\widehat{r}_\psi$, and preference dataset  $\mathcal{D} \leftarrow \emptyset$
    \STATE Initialize replay buffer $\gB$ and $\pi_\theta$ with unsupervised exploration
    \FOR{each iteration}
        \STATE Take action $\at \sim \pi_\theta$ and collect $s_{t+1}$ 
        % update reward model
        \IF{iteration \% $K == 0$}
            \STATE \small{\textcolor[HTML]{4b7223}{\texttt{// Query preference}}}
            \STATE Sample pair of trajectories $(\sigma^0,\sigma^1)$ and query human for $y$ 
            \STATE Store preference data into dataset $\mathcal{D} \leftarrow \mathcal{D} \cup \{ (\sigma^0,\sigma^1,y) \}$
            \STATE \small{\textcolor[HTML]{4b7223}{\texttt{// Reward learning}}}
            \STATE Sample batch $\{(\sigma^0,\sigma^1,y)_i\}^n_{i=1}$ from $\mathcal{D}$
            \STATE Optimize \cref{eq:reward_loss} to update $\widehat{r}_\psi$
            \STATE Relabel the replay buffer $\gB$ using $\widehat{r}_\psi$
        \ENDIF
        % Update the Graph
        \STATE \small{\textcolor[HTML]{4b7223}{\texttt{// Estimate conservative $\widehat{Q}$}}}
        \STATE Store transition $(\st, \at,\widehat{r}_\psi(\st, \at), s_{t+1})$ into replay buffer $\gB$
        \STATE Drive $\widehat{Q}$ via \cref{eq:value_iteration} (discrete setting)
        \STATE Update $Q_\xi$ via \cref{eq:loss_q_xi} (continuous setting)
        \STATE \small{\textcolor[HTML]{4b7223}{\texttt{// Policy regularization}}}
        \STATE Update $Q_\theta$ according to \cref{eq:discrete_q_loss}.(discrete setting)
        \STATE Update $Q_\theta$ and $\pi_\phi$ according to \cref{eq:q_loss} and \cref{eq:continuous_pi_loss}, respectively.(continuous setting)
    \ENDFOR
    \ENSURE policy $\pi_\phi$
\end{algorithmic} 
\end{algorithm}
\begin{algorithm}[]
\caption{\ourmethod (Offline)}
\label{pseudo_code_offline}
\begin{algorithmic}[1]
    \REQUIRE preference dataset $\gD$, dataset $\gB$
    \STATE Initialize parameters of $Q_\theta$, $\pi_\phi$, $\widehat{r}_\psi$
    \STATE \small{\textcolor[HTML]{4b7223}{\texttt{// Reward learning}}}
    \STATE Optimize \cref{eq:reward_loss} to update $\widehat{r}_\psi$ with preference data from $\mathcal{D}$
    \STATE Label the dataset $\gB$ via $\widehat{r}_\psi$
    \FOR{each iteration}
        \STATE Sample a batch $(s, a,\widehat{r}_\psi(s, a), s^\prime)$ from $\gB$
        \STATE \small{\textcolor[HTML]{4b7223}{\texttt{// Estimate conservative $\widehat{Q}$}}}
        \STATE Drive $\widehat{Q}$ via \cref{eq:value_iteration} (discrete setting)
        \STATE Update $Q_\xi$ via \cref{eq:loss_q_xi} (continuous setting)
        \STATE \small{\textcolor[HTML]{4b7223}{\texttt{// Policy regularization}}}
        \STATE Update $Q_\theta$ according to \cref{eq:discrete_q_loss}.(discrete setting)
        \STATE Update $Q_\theta$ and $\pi_\phi$ according to \cref{eq:q_loss} and \cref{eq:continuous_pi_loss}, respectively.(continuous setting)
    \ENDFOR
    \ENSURE policy $\pi_\phi$
\end{algorithmic} 
\end{algorithm}

\section{Experimental Details}\label{appendix:imp_details}
In this section, we provide the implementation details, including the basic settings for \PbRL, the architecture of the neural network, hyper-parameters, and other training details. For each run of experiments, we utilize one NVIDIA Tesla V100 GPU and 8 CPU cores for training.

\subsection{Number of Feedback}\label{appendix:num_feedback}
For the online settings, we use 100 preference pairs in Cheetah Run, Button Press, Window Open; 200 preference pairs in Walker Walk; 300 preference pairs in Push-5×5-1, Push-6x6-1, Push-7×7-1; 1000 preference pairs in Quadruped Walk, Push-5×5-2, Push-6x6-2, Push-7x7-2, Strip-shaped Building, Block-shaped Building, and Simple Two-Story Building tasks; 4000 preference pairs in Sweep Into. And for the offline settings, we use 100 preference pairs in antmaze-medium-play-v2, antmaze-medium-diverse-v2, hopper-medium-expert-v2, walker2d-medium-expert-v2, can-ph, lift-ph; 500 preference pairs in hopper-medium-replay-v2, walker2d-medium-replay-v2, can-mh, lift-mh; 1000 preference pairs in antmaze-large-play-v2, antmaze-large-diverse-v2.

\subsection{Basic Settings} \label{appendix:basic_pbrl}
In the following section, we provide more details of the unsupervised exploration and the uncertainty-based sampling scheme, both of which are mentioned in Section~\ref{subsec:setup}. These are pivotal techniques in enhancing the feedback efficiency of algorithms, as referenced in~\citet{lee2021pebble}. To ensure a fair comparison, all \PbRL algorithms in our experiments incorporate both unsupervised exploration and uncertainty-based sampling.

\noindent \textbf{Unsupervised Exploration.}
The technique of unsupervised exploration in {\PbRL} is proposed by~\citet{lee2021pebble}. Designing an intrinsic reward based on the entropy of the state efficiently encourages the agent to visit more diverse states and generate more various behaviors. More specifically, it uses a variant of particle-based entropy~\citep{10737616} as the estimation of entropy for the convenience of computation.

\noindent \textbf{Uncertainty-based Sampling.} There are some different sampling schemes, including but not limited to uniform sampling, disagreement sampling, and entropy sampling. The latter two sampling schemes are classified as uncertainty-based sampling, which has better performance compared to uniform sampling both intuitively and empirically. In our experiments, all method (online settings) use the disagreement sampling schemes.

\begin{table}[!ht]
\caption{Hyperparameters of SAC.}
\begin{center}
\resizebox{\textwidth}{!}{
\begin{tabular}{ll|ll}
\toprule
\textbf{Hyperparameter} & \textbf{Value} & \textbf{Hyperparameter} & \textbf{Value} \\
\midrule
Number of layers & 3 layers: 1 Conv2d, 2 Linear & Discount & 0.99 \\
Number of kernels of Conv2d & 16 & Batch size & 256 \\
Size of Kernel of Conv2d & 3 & Initial temperature & 0.2 \\
Stride of Conv2d & 1 & $(\beta_1,\beta_2)$ & (0.9,0.999) \\
Padding of Conv2d & 0 & Update freq & 4 \\
Hidden units of hidden layer & 128 & Critic target update freq & 8000 \\
Activation Function & ReLU & Critic $\tau$ & 1 \\
Actor optimizer & Adam & Exploration & 1 \\
Critic optimizer & Adam & Learning rate & 1e-4 \\
\bottomrule
\end{tabular}}
\end{center}
\label{table:hyperparameters_sac}
\end{table}

\subsection{Architecture and hyperparameters.}
In this section, we describe the architecture of neural networks in the SAC algorithm, which is used as the baseline method. Then we present the full list of hyperparameters of SAC, PEBBLE, and the proposed SEER. The actor of SAC has three layers; specifically, the first layer is the convolutional layer, composed of 16 kernels with a size of 3. Then we squeeze the output into one dimension as the input for the last two fully connected layers. The two Q networks of SAC have the same architecture as that of the actor: one convolutional layer and two fully connected layers. The detailed parameters of the neural network and hyperparameters during learning are shown in table ~\ref{table:hyperparameters_sac}. The hyperparameters of PEBBLE and SEER, which are different from those of SAC, are presented in table ~\ref{table:hyperparameters_pebble}.

\begin{table}[!htbp]
\caption{Hyperparameters of PEBBLE and SEER.}
\begin{center}
\resizebox{0.9\textwidth}{!}{
\begin{tabular}{ll|ll}
\toprule
\textbf{Hyperparameter} & \textbf{Value} & \textbf{Hyperparameter} & \textbf{Value} \\
\midrule
Length of segment      & 50      & Reward model ensemble size & 3 \\
Learning rate          & 0.0003  & Frequency of feedback      & 2000 \\
Reward batch size      & 128     & Number of train steps      & 1e6  \\
Reward update          & 200     & Replay buffer capacity     & 1e6  \\
Scaling parameter $\beta$ (SEER) & 6       & Graph update batch size (SEER) & 32 \\
Label smooth $\lambda$ (SEER)    & 0.05    & Regularizer weight $\eta$ (SEER)               & 6 \\
\bottomrule
\end{tabular}
}
\end{center}
\label{table:hyperparameters_pebble}
\end{table}

\subsection{Human labels}
We incorporate feedback from human subjects experienced in robotic tasks, as outlined in PT~\cite{kim2023preference}\footnote{\url{https://github.com/csmile-1006/PreferenceTransformer}}. Specifically, an informed human instructor evaluates each task by viewing video renderings of trajectory segments. They then determine which of the two segments more effectively facilitates the agent's goal achievement. Each segment spans 3 seconds, equivalent to 100 time steps. In instances where the human evaluator is unable to discern a preference between segments, they have the option to choose a neutral stance, attributing equal preference to both segments.

\section{Proof of Theorem}
\label{appendix:proof_theorem}

\subsection{Proof of Theorem~\ref{theorem:bellman}}
\label{appendix:proof}

In this section, we analyze the property of $\widehat Q$ in finite state-action space $\mathcal{S}\times\mathcal{A}$. 
% \begin{proof}[Proof Sketch.]
The proof of $\lim_{t\rightarrow \infty} Q_t = Q^*$ has been well-established in previous work \citep{robbins1951stochastic,jaakkola1993convergence,melo2001convergence}. Then the proof of $\lim_{t\rightarrow \infty} \widehat{Q}_t = \widehat{Q}^*$ is similar.
% can follow the proof of $Q_t$. 
We first prove the empirical Bellman operator~\cref{eq:empirical bellman operator} is a $\gamma$-contraction operator under the supremum norm. Then when updating in a sampling manner as \cref{eq:empirical bellman}, it can be considered as a random process. Borrowing an auxiliary result from stochastic approximation,
% \citep{robbins1951stochastic,jaakkola1993convergence}, 
we prove it satisfies the conditions that guarantee convergence. 
Finally, to prove $\widehat{Q}^*$ lower-bounds $Q^*$, we rewrite $\widehat{Q}^*(s,a) - Q^*(s,a)$ based on the standard and empirical Bellman operators. When the data covers the whole state-action space, we naturally have $\widehat{Q}^*=Q^*$.
% \end{proof}

For proof simplicity, we use $\beta$ denotes policies that interact with the environment and form the current replay memory.
We first show existing results for Bellman learning in \cref{eq:bellman}, and then prove \cref{theorem:bellman} in three steps.
The Bellman (optimality) operator $\mathcal{B}$ is defined as:
\begin{equation}
\label{eq:bellman operator}
(\mathcal{B}Q)(s,a) = \sum_{s^\prime\in \mathcal{S}}P(s^\prime|s,a)[r+\gamma \max_{a^\prime}Q(s^\prime,a^\prime)].
\end{equation}
Previous works have shown the operator $\mathcal{B}$ is a $\gamma$-contraction with respect to supremum norm:
\begin{equation*}
\Vert \mathcal{B}Q_1 - \mathcal{B}Q_2 \Vert_\infty \leq \gamma\Vert Q_1 - Q_2 \Vert_\infty,
\end{equation*}
the supremum norm $\Vert v \Vert_\infty = \max_{1\leq i \leq d} |v_i|$, $d$ is the dimension of vector $v$.
Following Banach's fixed-point theorem, $Q$ converges to optimal action value $Q^*$ if we consecutively apply operator $\mathcal{B}$ to $Q$,  $\lim_{n\rightarrow \infty} (\mathcal{B})^{n} Q = Q^*$.

Further, the update rule in \cref{eq:bellman}, i.e. $Q$-learning, is a sampling version that applies the $\gamma$-contraction operator $\mathcal{B}$ to $Q$. 
% The sampling version we use is:
\begin{equation}
Q(s,a) \leftarrow r(s,a)+\gamma\max_{a^\prime}Q(s^\prime,a^\prime). 
\label{eq:bellman}
\end{equation}
It can be considered as a random process and will converge to $Q^*$, $\lim_{t\rightarrow \infty}Q_t = Q^*$, with some mild conditions \citep{szepesvari2010algorithms,robbins1951stochastic,jaakkola1993convergence,melo2001convergence}.

Similarly, we define the empirical Bellman (optimality) operator $\mathcal{\hat{B}}$ as:
\begin{equation}
\label{eq:empirical bellman operator}
(\mathcal{\hat{B}}\widehat{Q})(s,a) = \sum_{s^\prime\in \mathcal{S}}P(s^\prime|s,a)[r+\gamma \max_{a^\prime:\beta(a^\prime|s^\prime)>0}\widehat{Q}(s^\prime,a^\prime)].
\end{equation}
And the sampling version we used on the graph is:
\begin{equation}
    \widehat{Q}(s,a) \leftarrow r+\gamma\max_{a^\prime:\beta(a^\prime|s^\prime)>0}\widehat{Q}(s^\prime,a^\prime),
\label{eq:empirical bellman}
\end{equation} 

We split \cref{theorem:bellman} into three lemmas.
We first show $\mathcal{\hat{B}}$ is a $\gamma$-contraction operator under supremum norm, thus converges to optimal action value $\widehat{Q}^*$, $\lim_{n\rightarrow \infty} (\mathcal{B})^{n} \widehat{Q} = \widehat{Q}^*$.
Then we show the sampling-based update rule in \cref{eq:empirical bellman} converges to $\widehat{Q}^*$, $\lim_{t\rightarrow \infty}\widehat{Q}_t = \widehat{Q}^*$. Finally, we show $\widehat{Q}^*$ lower-bounds $Q^*$, $\widehat{Q}^*(s,a)-Q^*(s,a)\leq 0, \forall (s,a) \in \mathcal{S}\times \mathcal{A}$. 
% And the equation holds when $\beta(a|s)>0$ for all state-action pairs.
And when the data covers the whole state-action space, i.e. $\beta(a|s)>0$ for all state-action pairs, we naturally have $\widehat{Q}^*(s,a)=Q^*(s,a)$.

\begin{lemma}
\label{lemma:empirical bellman contraction}
    The operator $\mathcal{\hat{B}}$ defined in \cref{eq:empirical bellman operator} is a $\gamma$-contraction operator under supremum norm,
\begin{equation*}
    \Vert \mathcal{\hat{B}}\widehat{Q}_1 - \mathcal{\hat{B}}\widehat{Q}_2 \Vert_\infty \leq \gamma\Vert \widehat{Q}_1 - \widehat{Q}_2 \Vert_\infty.
\end{equation*}
\end{lemma}
\begin{proof}
We can rewrite $\Vert \mathcal{\hat{B}}\widehat{Q}_1 - \mathcal{\hat{B}}\widehat{Q}_2 \Vert_\infty$ as
\begin{equation*}
\begin{split}
& \quad \Vert \mathcal{\hat{B}}\widehat{Q}_1 - \mathcal{\hat{B}}\widehat{Q}_2 \Vert_\infty  \\
&= \max_{s,a} \Big\lvert \sum_{s^\prime\in \mathcal{S}}P(s^\prime|s,a)[r + \gamma\max_{a^\prime_1:\beta(a^\prime_1|s^\prime)>0}\widehat{Q}_1(s^\prime,a^\prime_1)] - P(s^\prime|s,a)[r + \gamma\max_{a^\prime_2:\beta(a^\prime_2|s^\prime)>0}\widehat{Q}_2(s^\prime,a^\prime_2)] \Big\rvert \\ 
&= \max_{s,a}\gamma \Big\lvert \sum_{s^\prime\in \mathcal{S}}P(s^\prime|s,a)[\max_{a^\prime_1:\beta(a^\prime_1|s^\prime)>0}\widehat{Q}_1(s^\prime,a^\prime_1) - \max_{a^\prime_2:\beta(a^\prime_2|s^\prime)>0}\widehat{Q}_2(s^\prime,a^\prime_2)] \Big\rvert \\
&\leq \max_{s,a}\gamma \sum_{s^\prime\in \mathcal{S}}P(s^\prime|s,a) \Big\lvert \max_{a^\prime_1:\beta(a^\prime_1|s^\prime)>0}\widehat{Q}_1(s^\prime,a^\prime_1) - \max_{a^\prime_2:\beta(a^\prime_2|s^\prime)>0}\widehat{Q}_2(s^\prime,a^\prime_2) \Big\rvert \\
&\leq \max_{s,a}\gamma \sum_{s^\prime\in \mathcal{S}}P(s^\prime|s,a) \max_{\Tilde{a}:\beta(\Tilde{a}|s^\prime)>0} \Big\lvert \widehat{Q}_1(s^\prime,\Tilde{a})-\widehat{Q}_2(s^\prime,\Tilde{a})\Big\rvert \\
&\leq \max_{s,a}\gamma \sum_{s^\prime\in \mathcal{S}}P(s^\prime|s,a) \max_{\Tilde{s},\Tilde{a}:\beta(\Tilde{a}|\Tilde{s})>0} \Big\lvert \widehat{Q}_1(\Tilde{s},\Tilde{a})-\widehat{Q}_2(\Tilde{s},\Tilde{a})\Big\rvert \\
&= \max_{s,a}\gamma \sum_{s^\prime\in \mathcal{S}}P(s^\prime|s,a) \Vert \widehat{Q}_1 - \widehat{Q}_2 \Vert_\infty \\
&= \gamma \Vert \widehat{Q}_1 - \widehat{Q}_2 \Vert_\infty,
\end{split}
\end{equation*}
where the last line follows from $\sum_{s^\prime\in \mathcal{S}}P(s^\prime|s,a)=1$.
\end{proof}
\vspace{0.2cm}

To show the sampling-based update rule in \cref{eq:empirical bellman} converges to $\widehat{Q}^*$, we borrow an auxiliary result from stochastic approximation \citep{robbins1951stochastic,jaakkola1993convergence}.
\vspace{0.2cm}

\begin{theorem}
\label{theorem:stochastic approximation}
The random process $\{\Delta_t\}$ taking values in $\mathbb{R}^n$ and defined as
\begin{equation}
\label{eq:random process}
    \Delta_{t+1}(x) = (1-\alpha_t(x))\Delta_t(x) + \alpha_t(x)F_t(x)
\end{equation}
converges to zero w.p.1 under the following assumptions:

\hspace*{1em} (1) $0 \leq \alpha_t \leq 1, \sum_t \alpha_t(x)=\infty$ and $\sum_t\alpha_t^2(x)<\infty$;

\hspace*{1em} (2)  $\Vert \mathbb{E}[F_t(x)|\mathcal{F}_t] \Vert_W \leq \gamma \Vert \Delta_t \Vert_W$, with $\gamma <1$;

\hspace*{1em} (3) $Var[F_t(x)|\mathcal{F}_t] \leq C(1+ \Vert \Delta_t \Vert_W^2)$, for $C>0$.
\end{theorem}
$W$ is a norm. In our proof it is a supremum norm.
% \begin{equation*}
%     \Vert v \Vert_\infty = \max_{1\leq i \leq d} |v_i|,
% \end{equation*}
% $d$ is the dimension of vector $v$.
\begin{proof}
See~\citet{robbins1951stochastic,jaakkola1993convergence}.
\end{proof}
\vspace{0.2cm}

\begin{lemma}
\label{lemma:empirical q learning}
Given any initial estimation $\widehat{Q}_0$, the following update rule:
\begin{equation}
\label{eq:q-learning empirical bellman}
    \widehat{Q}_{t+1}(s_t,a_t) = \widehat{Q}_t(s_t,a_t) + \alpha_t(x_t,a_t)[r_t + \gamma \max_{a:\beta(a|s_{t+1})>0}\widehat{Q}_t(s_{t+1},a)-\widehat{Q}_t(s_t,a_t)],
\end{equation}
converges w.p.1 to the optimal action-value function $\widehat{Q}^*$ if
\begin{equation*}
0 \leq \alpha_t(s,a) \leq 1, \quad \sum_t \alpha_t(s,a)=\infty \quad and \quad \sum_t\alpha_t^2(s,a)<\infty,
\end{equation*}
for all $(s,a) \in \mathcal{S}\times \mathcal{A}$.
\end{lemma}

% Then when using sampling, it can be considered as a random process. We prove it satisfies the conditions that guarantee the random process converges with probability 1. To prove $\widehat{Q}^*$ lower-bounds $Q^*$, we can derive  $Q^*(s,a)-\widehat{Q}^*(s,a)\ge 0$ from~\cref{eq:bellman} and~\cref{eq:empirical bellman}.
% Since the Empirical Bellman operator $\mathcal{\hat{B}}$ is for the whole state-action space, and \cref{eq:empirical bellman} uses point samples. 
% The update rule can be written as:
% where the step-sizes $0 \leq \alpha_t(s,a) \leq 1$. This means at each update, only one state-action pair is updated. To prove this update converges, we need an auxiliary result from stochastic approximation~\citep{robbins1951stochastic,jaakkola1993convergence}.

\begin{proof}
Based on \cref{theorem:stochastic approximation}, we prove the update rule in \cref{eq:q-learning empirical bellman} converges. 

Rewrite \cref{eq:q-learning empirical bellman} as
\begin{equation*}
    \widehat{Q}_{t+1}(s_t,a_t) = (1-\alpha_t(s_t,a_t))\widehat{Q}_t(s_t,a_t) + \alpha_t(x_t,a_t)[r_t + \gamma \max_{a:\beta(a|s_{t+1})>0}\widehat{Q}_t(s_{t+1},a)]
\end{equation*}
Subtract $\widehat{Q}^*(s_t,a_t)$ from both sides:
\begin{equation*}
\begin{split}
& \quad \widehat{Q}_{t+1}(s_t,a_t) - \widehat{Q}^*(s_t,a_t) \\ 
&= (1-\alpha_t(s_t,a_t))(\widehat{Q}_t(s_t,a_t)-\widehat{Q}^*(s_t,a_t)) + \alpha_t(x_t,a_t)[r_t + \gamma \max_{a:\beta(a|s_{t+1})>0}\widehat{Q}_t(s_{t+1},a)-\widehat{Q}^*(s_t,a_t)]
\end{split}
\end{equation*}
Let 
\begin{equation}
\label{eq:delta}
    \Delta_t(s,a)=\widehat{Q}(s,a)-\widehat{Q}^*(s,a)
\end{equation}
and 
\begin{equation}
\label{eq:F}
    F_t(s,a) = r + \gamma \max_{a^\prime:\beta(a^\prime|s^\prime)>0}\widehat{Q}_t(s^\prime,a^\prime)-\widehat{Q}^*(s,a).
\end{equation}
We get the same random process shown in \cref{theorem:stochastic approximation} \cref{eq:random process}.
Then, proving $\lim_{t\rightarrow \infty} \widehat{Q}_t = \widehat{Q}^*$ is the same as proving $\Delta_t(s,a)$ converges to zero with probability 1. We only need to show the assumptions in \cref{theorem:stochastic approximation} are satisfied under the definitions of \cref{eq:delta,eq:F}.

\cref{theorem:stochastic approximation} (1) is the same as the condition in \cref{lemma:empirical q learning}. It is easy to achieve, for example, we can choose $\alpha_t(s,a) = 1/t$.

For \cref{theorem:stochastic approximation} (2), we have
\begin{equation*}
\begin{split}
\mathbb{E}[F_t(s,a)|\mathcal{F}_t] &= \sum_{s^\prime \in \mathcal{S}}P(s^\prime|s,a)[r+\gamma\max_{a^\prime:\beta(a^\prime|s^\prime)}\widehat{Q}_t(s^\prime,a^\prime)-\widehat{Q}^*(s,a)] \\
&= (\mathcal{\hat{B}}\widehat{Q}_t)(s,a)-\widehat{Q}^*(s,a) \\
&= (\mathcal{\hat{B}}\widehat{Q}_t)(s,a)-(\mathcal{\hat{B}}\widehat{Q}^*)(s,a)
\end{split}
\end{equation*}
Thus,
\begin{equation*}
\begin{split}
\Vert \mathbb{E}[F_t(s,a)|\mathcal{F}_t] \Vert_\infty &= \Vert (\mathcal{\hat{B}}\widehat{Q}_t)-(\mathcal{\hat{B}}\widehat{Q}^*) \Vert_\infty \\
&\leq \gamma \Vert \widehat{Q}_t - \widehat{Q}^* \Vert_\infty \\
&= \gamma \Vert \Delta_t \Vert_\infty,
\end{split}
\end{equation*}
with $\gamma <1$. 

For \cref{theorem:stochastic approximation} (3), we have
\begin{equation*}
\begin{split}
Var[F_t(s)|\mathcal{F}_t] &= \mathbb{E}[F_t(s)-\mathbb{E}[F_t(s)|\mathcal{F}_t]|\mathcal{F}_t]^2 \\
&= \mathbb{E}[F_t(s)-((\mathcal{\hat{B}}\widehat{Q}_t)(s,a)-(\mathcal{\hat{B}}\widehat{Q}^*)(s,a))]^2 \\
&= \mathbb{E}[r + \gamma \max_{a^\prime:\beta(a^\prime|s^\prime)>0}\widehat{Q}_t(s^\prime,a^\prime)-\widehat{Q}^*(s,a)-((\mathcal{\hat{B}}\widehat{Q}_t)(s,a)-(\mathcal{\hat{B}}\widehat{Q}^*)(s,a))]^2 \\
&= \mathbb{E}[r + \gamma \max_{a^\prime:\beta(a^\prime|s^\prime)>0}\widehat{Q}_t(s^\prime,a^\prime) - (\mathcal{\hat{B}}\widehat{Q}_t)(s,a)]^2 \\
&= Var[r + \gamma \max_{a^\prime:\beta(a^\prime|s^\prime)>0}\widehat{Q}_t(s^\prime,a^\prime)|\mathcal{F}_t]
\end{split}
\end{equation*}
We add and minus a $\widehat{Q}^*$ term to make it close to the RHS in \cref{theorem:stochastic approximation} (3):
\begin{equation*}
    Var[r + \gamma \max_{a^\prime:\beta(a^\prime|s^\prime)>0}\widehat{Q}^*(s^\prime,a^\prime) + \gamma \max_{a^\prime:\beta(a^\prime|s^\prime)>0}\widehat{Q}_t(s^\prime,a^\prime) - \gamma \max_{a^\prime:\beta(a^\prime|s^\prime)>0}\widehat{Q}^*(s^\prime,a^\prime)|\mathcal{F}_t]
\end{equation*}

Since $r$ is bounded, thus $r + \gamma \max_{a^\prime:\beta(a^\prime|s^\prime)>0}\widehat{Q}^*(s^\prime,a^\prime)$ is bounded.
And clearly the second part
$\max_{a^\prime:\beta(a^\prime|s^\prime)>0}\widehat{Q}_t(s^\prime,a^\prime) - \max_{a^\prime:\beta(a^\prime|s^\prime)>0}\widehat{Q}^*(s^\prime,a^\prime)$
can be bounded by $\Vert \Delta_t \Vert_\infty$ with some constant.  Thus, we have 
\begin{equation*}
Var[F_t(s)|\mathcal{F}_t] \leq  C(1+ \Vert \Delta_t \Vert_\infty^2),
\end{equation*}
for some constant $C>0$ under supremum norm.
Thus, by \cref{theorem:stochastic approximation}, $\Delta_t$ converges to zero w.p.1, i.e., $\widehat{Q}_t$ converges to $\widehat{Q}^*$ w.p.1.
\end{proof}
\vspace{0.2cm}

\begin{lemma}
\label{lemma:lower bound}
The value estimation obtained by \cref{eq:empirical bellman operator} lower-bounds the value estimation obtained by \cref{eq:bellman operator}:
\begin{equation}
\widehat{Q}^*(s,a)-Q^*(s,a)\leq 0
\end{equation}
for all state-action pairs.
\end{lemma}

\begin{proof}
Following the definition of \cref{eq:bellman operator,eq:empirical bellman operator}, we can rewrite as
\begin{equation*}
\begin{split}
& \quad \max_{s,a} (\widehat{Q}^*(s,a)-Q^*(s,a)) \\
&= \max_{s,a} (\hat{\mathcal{B}}\widehat{Q}^*(s,a)-\mathcal{B}Q^*(s,a)) \\ 
&= \max_{s,a} (\sum_{s^\prime\in \mathcal{S}}P(s^\prime|s,a)[r+\gamma \max_{\hat{a}^\prime:\beta(\hat{a}^\prime|s^\prime)>0}\widehat{Q}^*(s^\prime,\hat{a}^\prime)] - \sum_{s^\prime\in \mathcal{S}}P(s^\prime|s,a)[r+\gamma \max_{a^\prime}Q^*(s^\prime,a^\prime)]) \\
&= \max_{s,a}\sum_{s^\prime\in \mathcal{S}}P(s^\prime|s,a)\gamma (\max_{\hat{a}^\prime:\beta(\hat{a}^\prime|s^\prime)>0}\widehat{Q}^*(s^\prime,\hat{a}^\prime) - \max_{a^\prime}Q^*(s^\prime,a^\prime)) \\
&\leq \max_{s,a}\sum_{s^\prime\in \mathcal{S}}P(s^\prime|s,a)\gamma (\max_{\hat{a}^\prime}\widehat{Q}^*(s^\prime,\hat{a}^\prime) - \max_{a^\prime}Q^*(s^\prime,a^\prime)) \\
&\leq \max_{s,a}\sum_{s^\prime\in \mathcal{S}}P(s^\prime|s,a)\gamma \max_{\Tilde{a}}(\widehat{Q}^*(s^\prime,\Tilde{a}) - Q^*(s^\prime,\Tilde{a})) \\
&\leq \max_{s,a}\gamma\sum_{s^\prime\in \mathcal{S}}P(s^\prime|s,a) \max_{\Tilde{s},\Tilde{a}}(\widehat{Q}^*(\Tilde{s},\Tilde{a}) - Q^*(\Tilde{s},\Tilde{a})) \\
&= \gamma \max_{\Tilde{s},\Tilde{a}}(\widehat{Q}^*(\Tilde{s},\Tilde{a}) - Q^*(\Tilde{s},\Tilde{a})) = \gamma \max_{s,a} (\widehat{Q}^*(s,a)-Q^*(s,a))
\end{split}
\end{equation*}
where the last line follows from $\sum_{s^\prime\in \mathcal{S}}P(s^\prime|s,a)=1$. Then we have
\begin{equation*}
\begin{split}
\max_{s,a} (\widehat{Q}^*(s,a)-Q^*(s,a)) &\leq \gamma \max_{s,a} (\widehat{Q}^*(s,a)-Q^*(s,a)) \\
&\leq \gamma^2 \max_{s,a} (\widehat{Q}^*(s,a)-Q^*(s,a)) \\
&\leq \cdots \\
&\leq \gamma^n \max_{s,a} (\widehat{Q}^*(s,a)-Q^*(s,a))
\end{split}
\end{equation*}
Take limit for both sides and since $0<\gamma<1$, we have $\max_{s,a} (\widehat{Q}^*(s,a)-Q^*(s,a)) \leq 0$.

When $\beta(a|s)>0$ for all state-action pairs, the two contraction operators $\mathcal{\hat{B}}$ and $\mathcal{B}$ are the same. 
% Since $\mathcal{\hat{B}}$ and $\mathcal{B}$ are contractions, 
And based on Banach's fixed-point theorem, there is a unique fixed point.
% $\mathcal{\hat{B}}$ and $\mathcal{B}$ have unique fixed point, 
Thus $\widehat{Q}^*(s,a)=Q^*(s,a)$ for all state-action pairs., i.e., $\widehat{Q}^*(s,a)-Q^*(s,a)=0, (s,a) \in \mathcal{S}\times\mathcal{A}$ holds when $\beta(a|s)>0$ for all state-action pairs.
\end{proof}
Then, we get \cref{theorem:bellman} proved with \cref{lemma:empirical bellman contraction,lemma:empirical q learning,lemma:lower bound}.

\section{Environment Specifications} \label{appendix:task_specifications}
In this section, we delineate the tasks utilized in our experiments. For online settings, discrete tasks include Sokoban~\citep{SchraderSokoban2018} and Craftenv~\citep{ZhaoCraftEnv2023}, while continuous tasks encompass robotic manipulation challenges from Meta-world~\cite{yu2020meta} and locomotion tasks from the DeepMind Control Suite (DMControl)~\cite{tassa2018deepmind, tunyasuvunakool2020dm_control}. Regarding offline settings, we incorporate control tasks from the D4RL benchmarks~\citep{fu2020d4rl}.

\subsection{Sokoban}\label{appendix:sokoban_intro}
Sokoban~\citep{SchraderSokoban2018}, the Japanese word for 'a warehouse keeper', is a puzzle video game, which is analogous to the problem of having an agent in a warehouse push some specified boxes from their initial locations to target locations. Target locations have the same number of boxes. The goal of the game is to manipulate the agent to move all boxes to the target locations. Specifically, the game is played on a rectangular grid called a room, and each cell of the room is either a floor or a wall. At each new episode, the environment will be reset, which means the layout of the room is randomly generated, including the floors, the walls, the target locations, the boxes' initial locations, and the location of the agent. We choose six tasks with different complexities from Push-5×5-1 to Push-7×7-2, which is shown in Figure~\ref{fig:sokoban}. The numbers in the task name denote respectively the size of the grid and the number of boxes. 

\noindent \textbf{State Space.} The state space consists of all possible images displayed on the screen. Each image has the same size as the map, and using the way of dividing each pixel of the image by 255 to normalize into [0,1], we preprocess the image to the inputting state. 

\noindent \textbf{Action Space.} The action space of Sokoban has a total of eight actions, composed of moving and pushing the box in four directions, which are \textit{left}, \textit{right}, \textit{up}, \textit{down}, \textit{push-left}, \textit{push-right}, \textit{push-up}, \textit{push-down} in detail.

\noindent \textbf{Reward Setting.} The agent gets a punishment with a -0.1 reward after each time step. Successfully pushing a box to the target location, can get a +1 reward, and if all boxes are laid in the right locations, the agent can obtain an extra +10 reward. We set the max episode steps to 120, which means the cumulative reward during one episode ranges from -12 to 10 plus the number of boxes.

\begin{figure*}[!ht]
\centering
\begin{tabular}{cccc}
\hspace*{-1.1em} \subfloat[Push-5x5-1]{\includegraphics[width=0.25\linewidth]{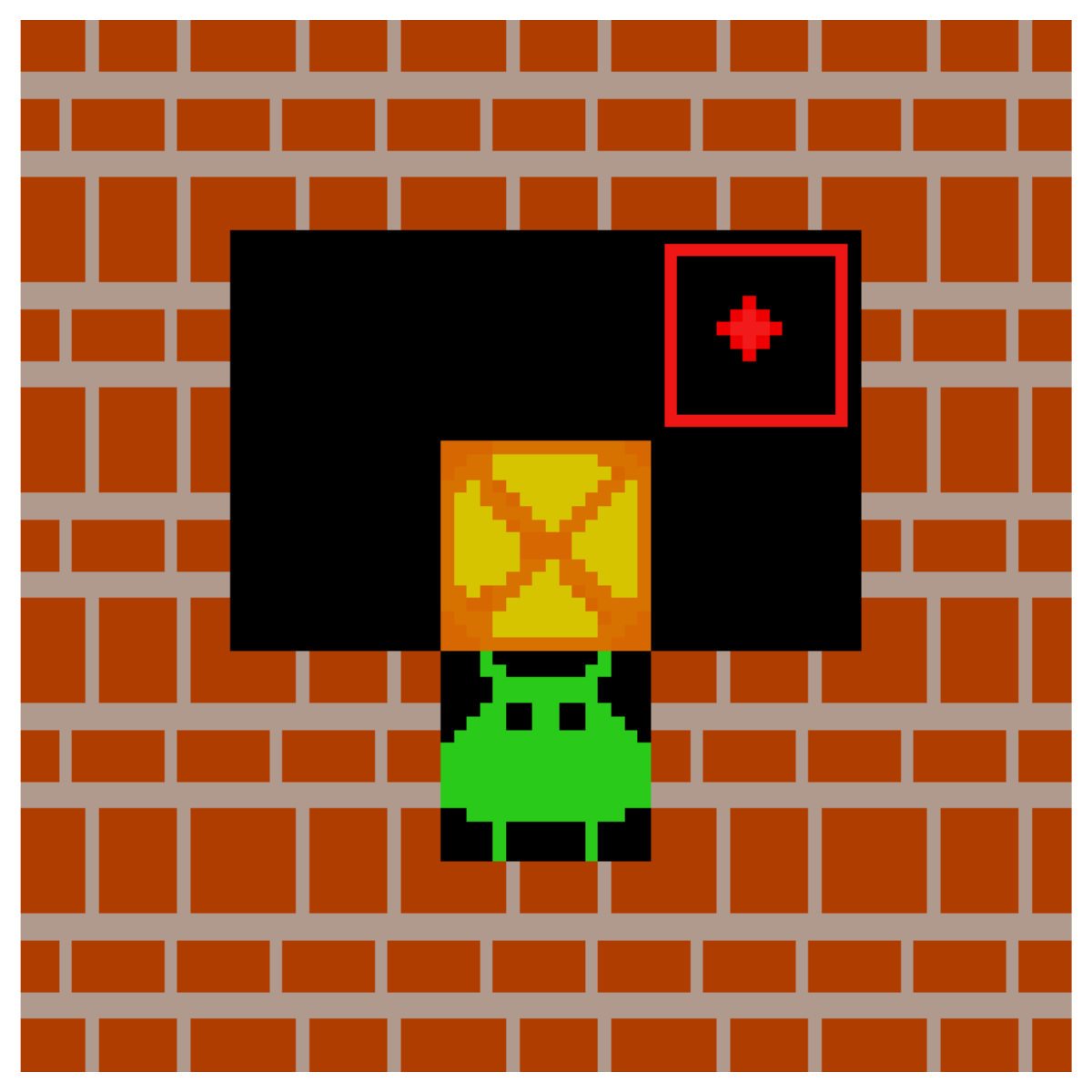}}
& \hspace*{-1.1em} \subfloat[Push-5x5-2]{\includegraphics[width=0.25\linewidth]{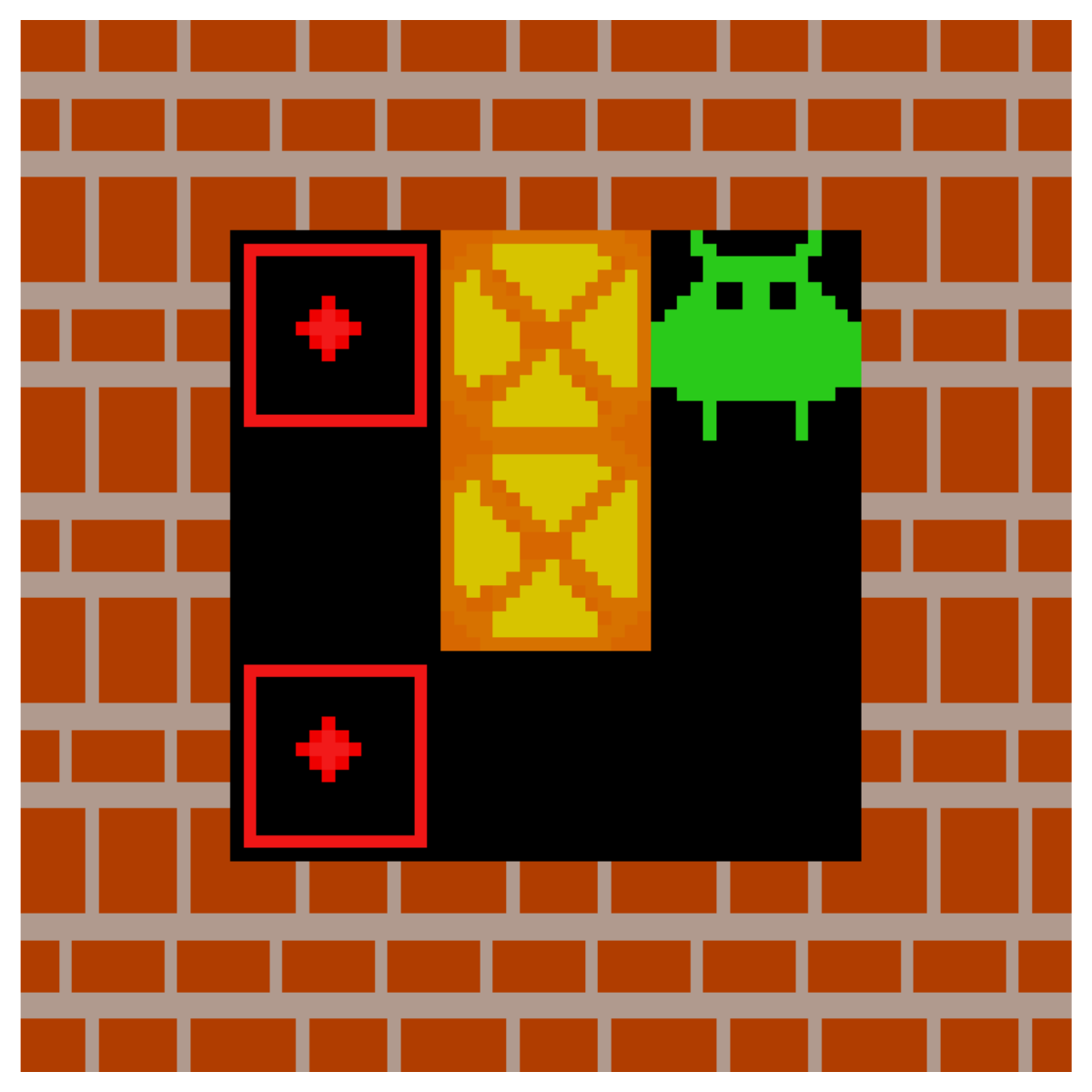}}
& \hspace*{-1.1em} \subfloat[Push-6x6-1]{\includegraphics[width=0.25\linewidth]{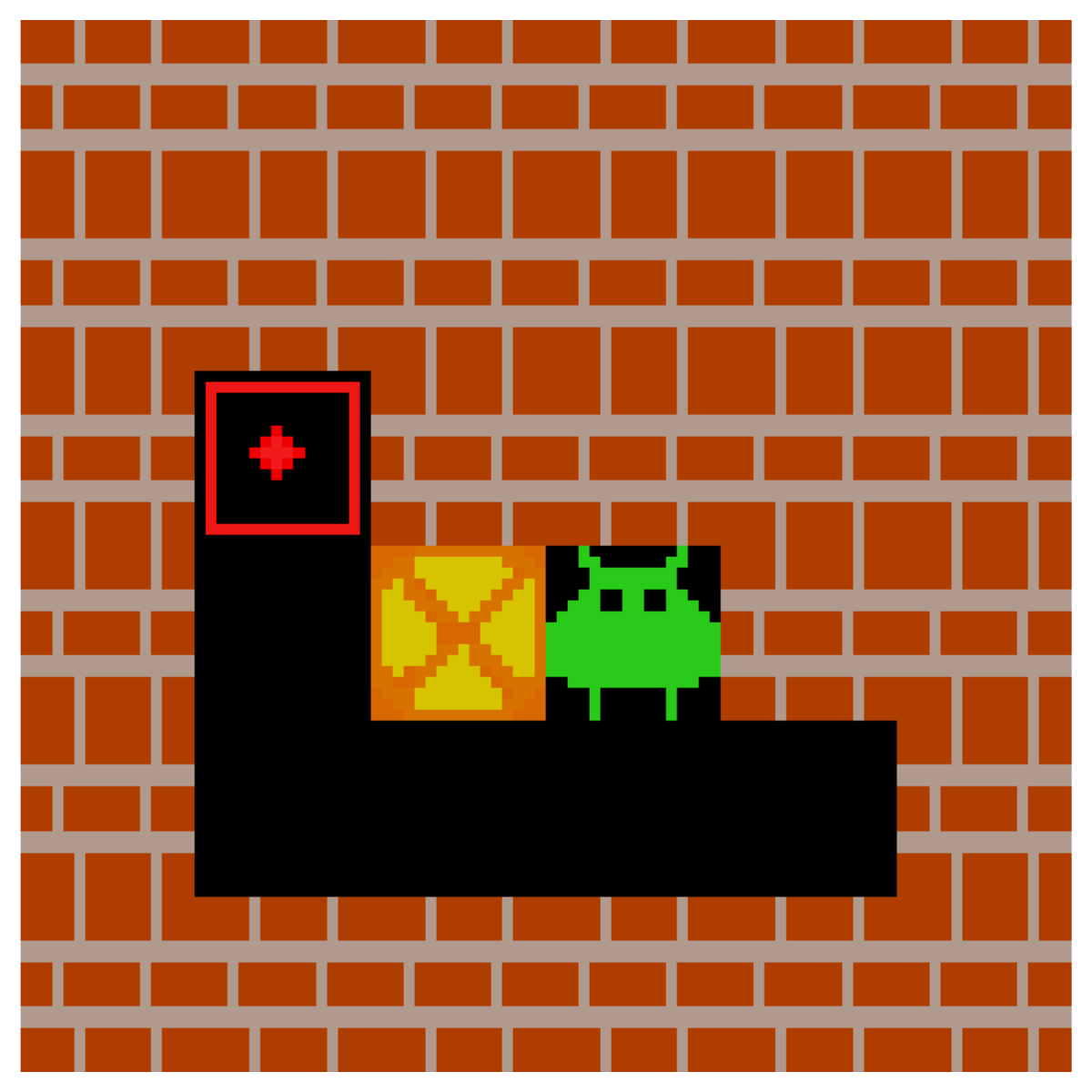}}
& \hspace*{-1.1em} \subfloat[Push-6x6-2]{\includegraphics[width=0.25\linewidth]{figures/envs/Sokoban661.pdf}}
\end{tabular}
\caption{Visualization of puzzle tasks from Sokoban, which focuses on evaluating the capabilities of agents in spatial reasoning, logical deduction, and long-term planning.}
\label{fig:sokoban}
\end{figure*}

\subsection{CraftEnv}\label{appendix:craftenv_intro}
Craftenv~\citep{ZhaoCraftEnv2023}, A Flexible Robotic Construction Environment, is a collection of construction tasks. The agent needs to learn to manipulate the elements, including smartcar, blocks, and slopes, to achieve a target structure through efficient and effective policy. Each construction task is a simulation of the corresponding complex real-world task, which is challenging enough for reinforcement learning algorithms. Meanwhile, the CraftEnv is highly malleable, enabling researchers to design their own tasks for specific requirements. The environment is simple to use since it is implemented by Python and can be rendered using PyBullet. We choose three different designs of the building tasks, shown in Figure~\ref{fig:craftenv}, to evaluate our algorithm in CraftEnv.

\noindent \textbf{State Space.} We assume that the agent can obtain all the information in the map. Therefore, the state consists of all knowledge of smartcar, blocks, folded slopes, unfolded slopes' body, and unfolded slopes' foot, including the position and the yaw angle.

\noindent \textbf{Action Space.} The available actions of an agent are designed based on real-world smartcar models, including a total of fifteen actions. Besides all eight directions moving actions, i.e. \textit{forward}, \textit{backward}, \textit{left}, \textit{right}, \textit{left-forward}, \textit{left-backward}, \textit{right-forward}, and \textit{right-backward}, there are interaction-related actions, designed to simulate the building process in the real world. Specifically, the agent can act \textit{lift} and \textit{drop} actions to decide whether or not to carry the surrounding basic element, and can \textit{flod} or \textit{unflod} slopes to build the complex buildings. In addition, the actions of \textit{rotate-left} and \textit{rotate-right} control the agent to rotate the main body to the left and right, and \textit{stop} action is just a non-action.

\noindent \textbf{Reward Setting.} CraftEnv is a flexible environment as mentioned above. We can specify our own reward function for different construction tasks. For the relatively simple tasks of building with specified shape requirements, we can use discrete reward, where some reward is given when part of the blueprint is built. While, for building tasks with high complexity, various reward patterns should be designed to encourage the agent to build with different intentions.

\begin{figure*}[!ht]
\centering
\begin{tabular}{cccc}
\hspace*{-1.1em} \subfloat[The Strip-shaped Building]{\includegraphics[width=0.33\linewidth]{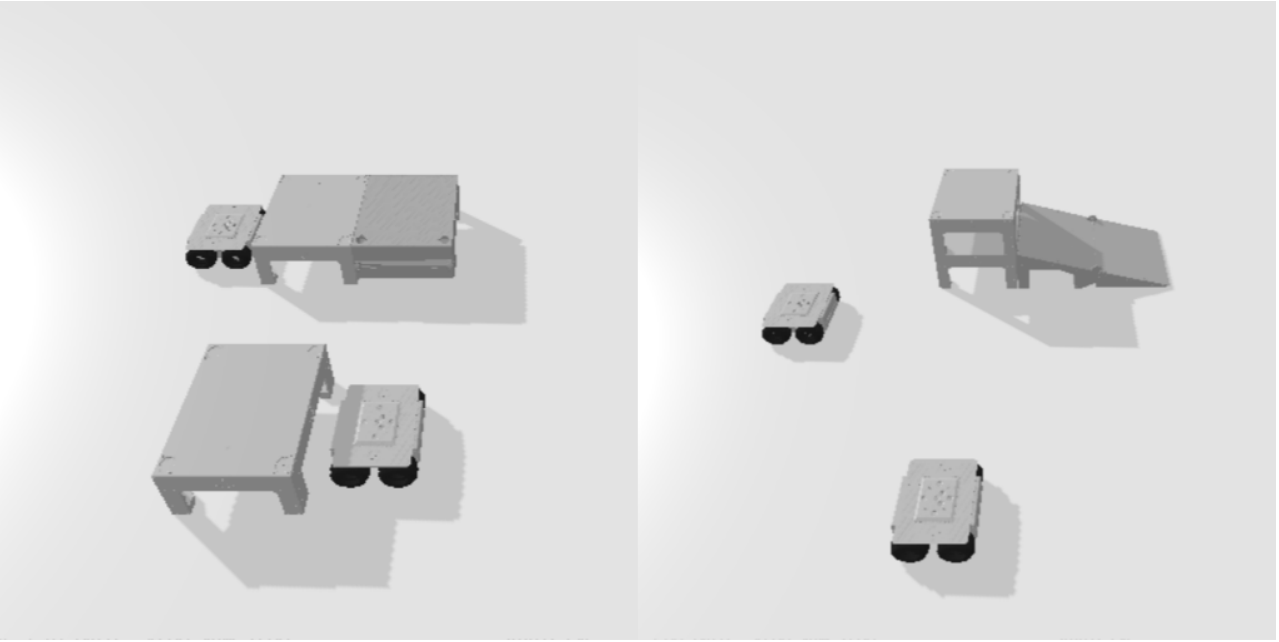}}
& \hspace*{-1.1em} \subfloat[The Block-shaped Building]{\includegraphics[width=0.33\linewidth]{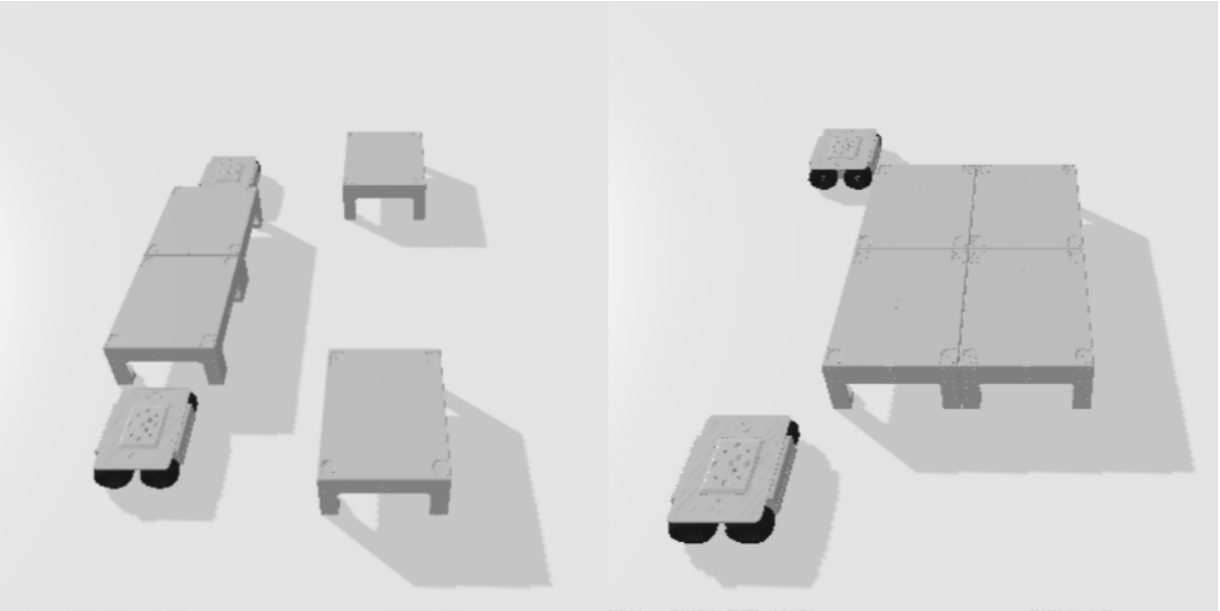}}
& \hspace*{-1.1em} \subfloat[The Simple Two-Story Building]{\includegraphics[width=0.33\linewidth]{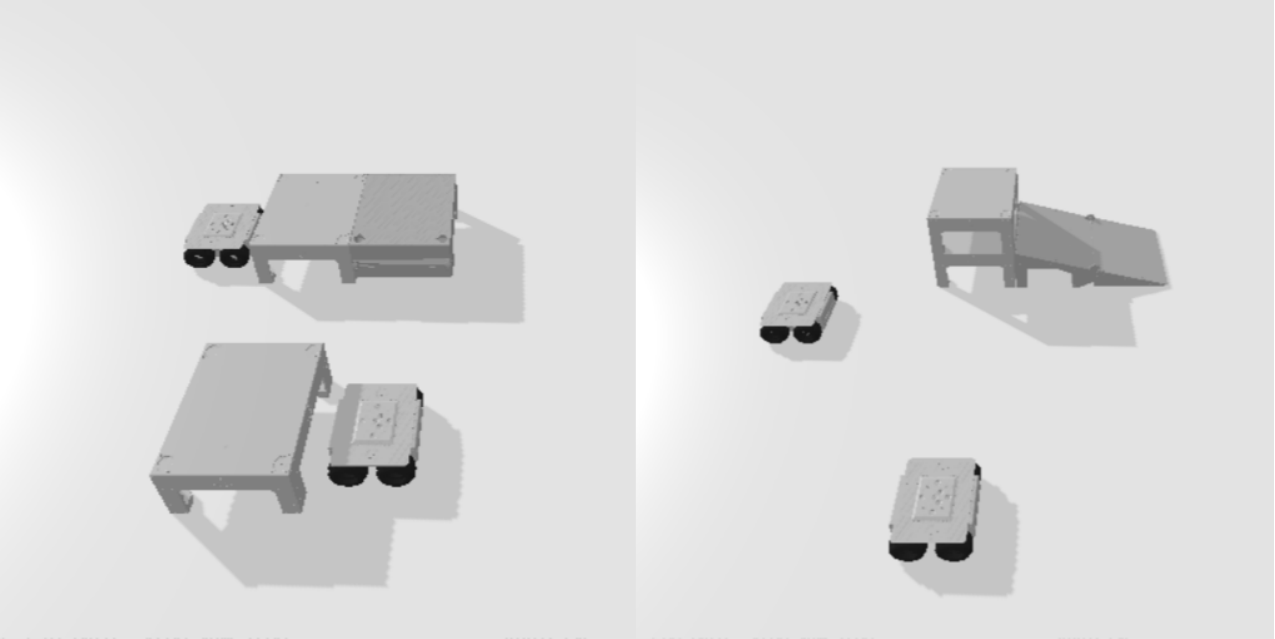}}
\end{tabular}
\caption{Visualization of building tasks from CraftEnv. From left to right are The Strip-shaped Building, The Block-shaped Building, and The Simple Two-Story Building task respectively.}
\label{fig:craftenv}
\end{figure*}

\subsection{Robotic tasks}
In our experiments, we utilize robotic manipulation tasks from Meta-world~\cite{yu2020meta} and locomotion tasks from the DeepMind Control Suite (DMControl)~\cite{tassa2018deepmind, tunyasuvunakool2020dm_control}, which is visualized in Figure~\ref{fig:render_tasks}. Meta-World comprises 50 diverse manipulation tasks with a common structural framework, while DMControl offers a collection of stable, rigorously tested tasks, serving as benchmarks for continuous control learning agents. These tasks are equipped with well-formulated reward functions that facilitate agent learning. For performance evaluation, Meta-world introduces an interpretable success metric for each task, which serves as the standard evaluation criterion across various settings. This metric often hinges on the proximity of task-relevant objects to their final goal positions, quantified as $||o-g|| \leq \epsilon$, with $\epsilon$ representing a nominal distance threshold, such as 5 cm. DMControl, conversely, employs episode returns for evaluation. We adopt fixed-length episodes of 1000 timesteps as a proxy. Given that all reward functions are designed such that $r \approx 1$ is in proximity to goal states, learning curves that measure total returns can maintain consistent y-axis limits of $[0,1000]$. This uniformity simplifies interpretation and facilitates averaging across tasks.

\begin{figure*}[!t]
\centering
\begin{tabular}{ccc}
\subfloat[Cheetah]{\includegraphics[width=0.27\linewidth, height=0.18\linewidth]{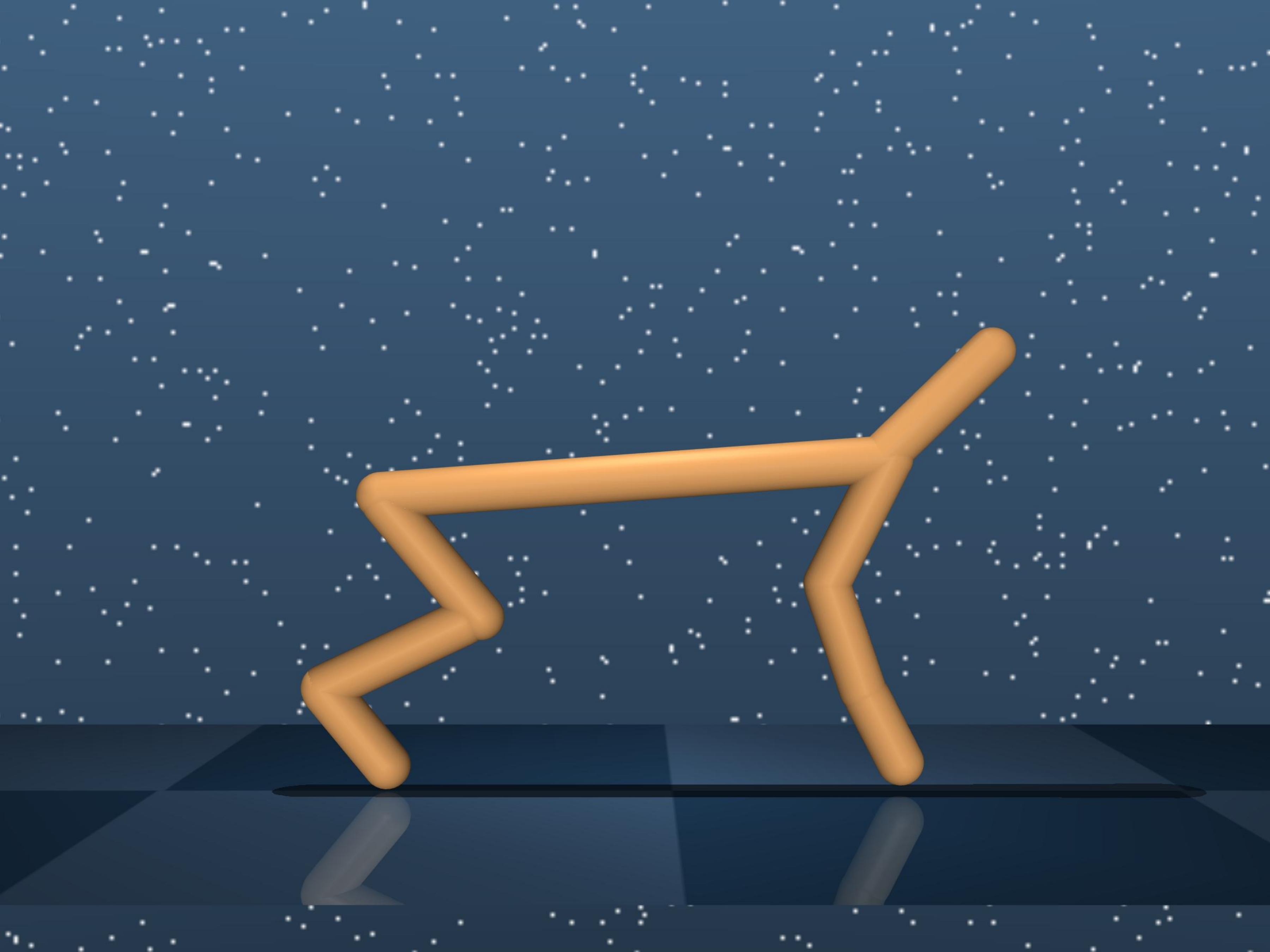}}
& \subfloat[Walker]{\includegraphics[width=0.27\linewidth, height=0.18\linewidth]{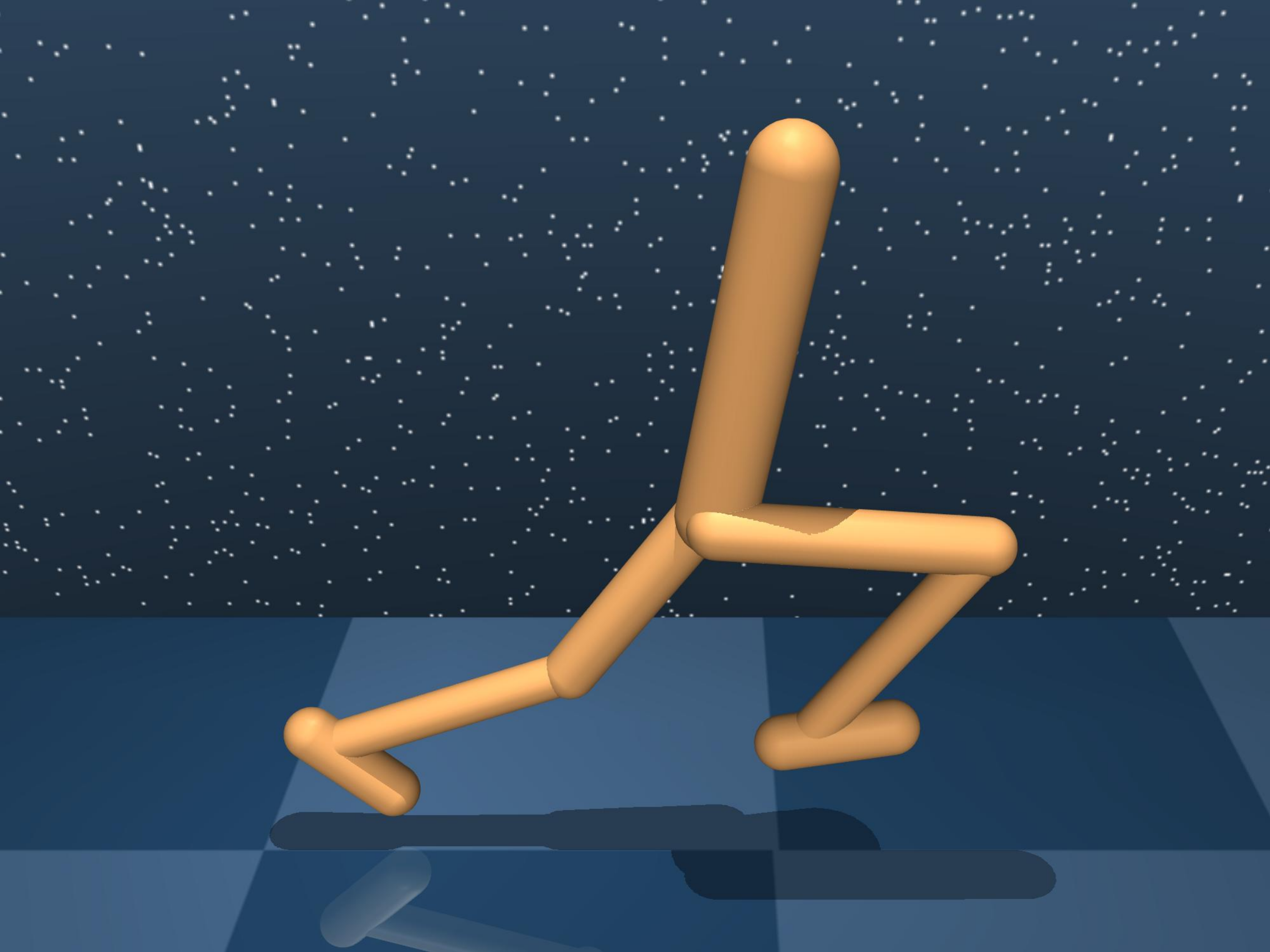}}
& \subfloat[Quadruped]{\includegraphics[width=0.27\linewidth, height=0.18\linewidth]{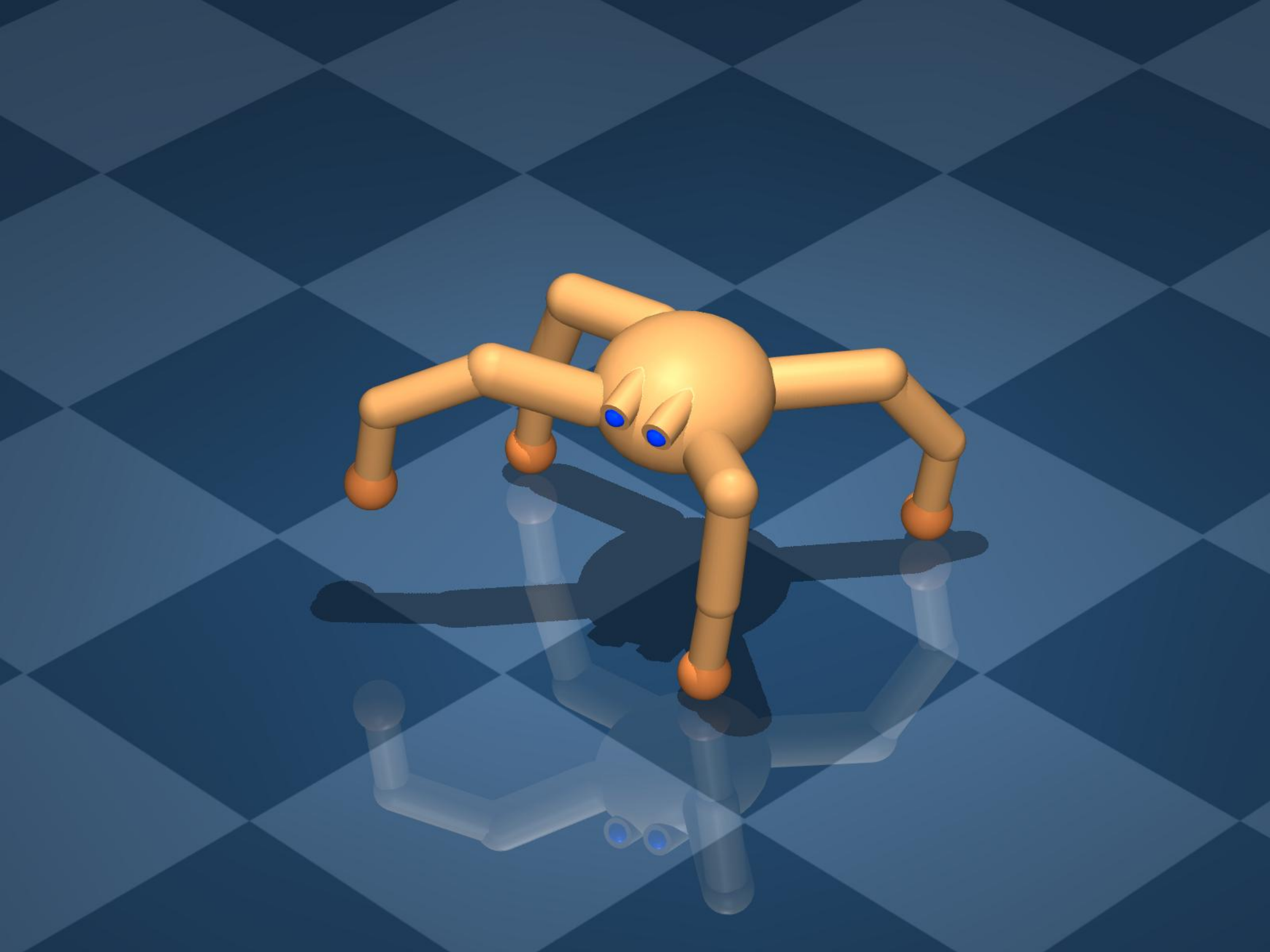}} \\
\subfloat[Button Press]{\includegraphics[width=0.27\linewidth, height=0.18\linewidth]
{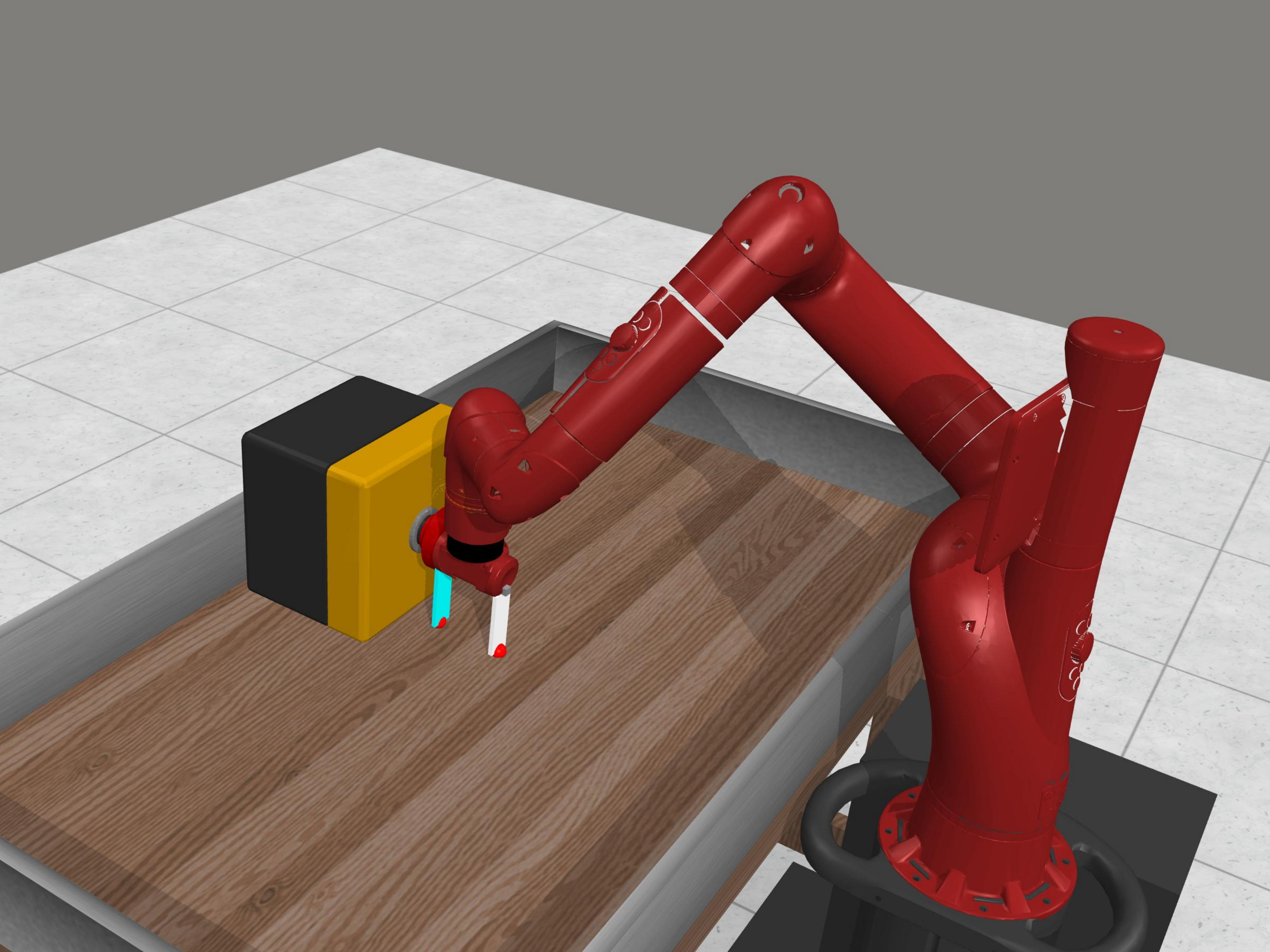}}
& \subfloat[Window Open]{\includegraphics[width=0.27\linewidth, height=0.18\linewidth]{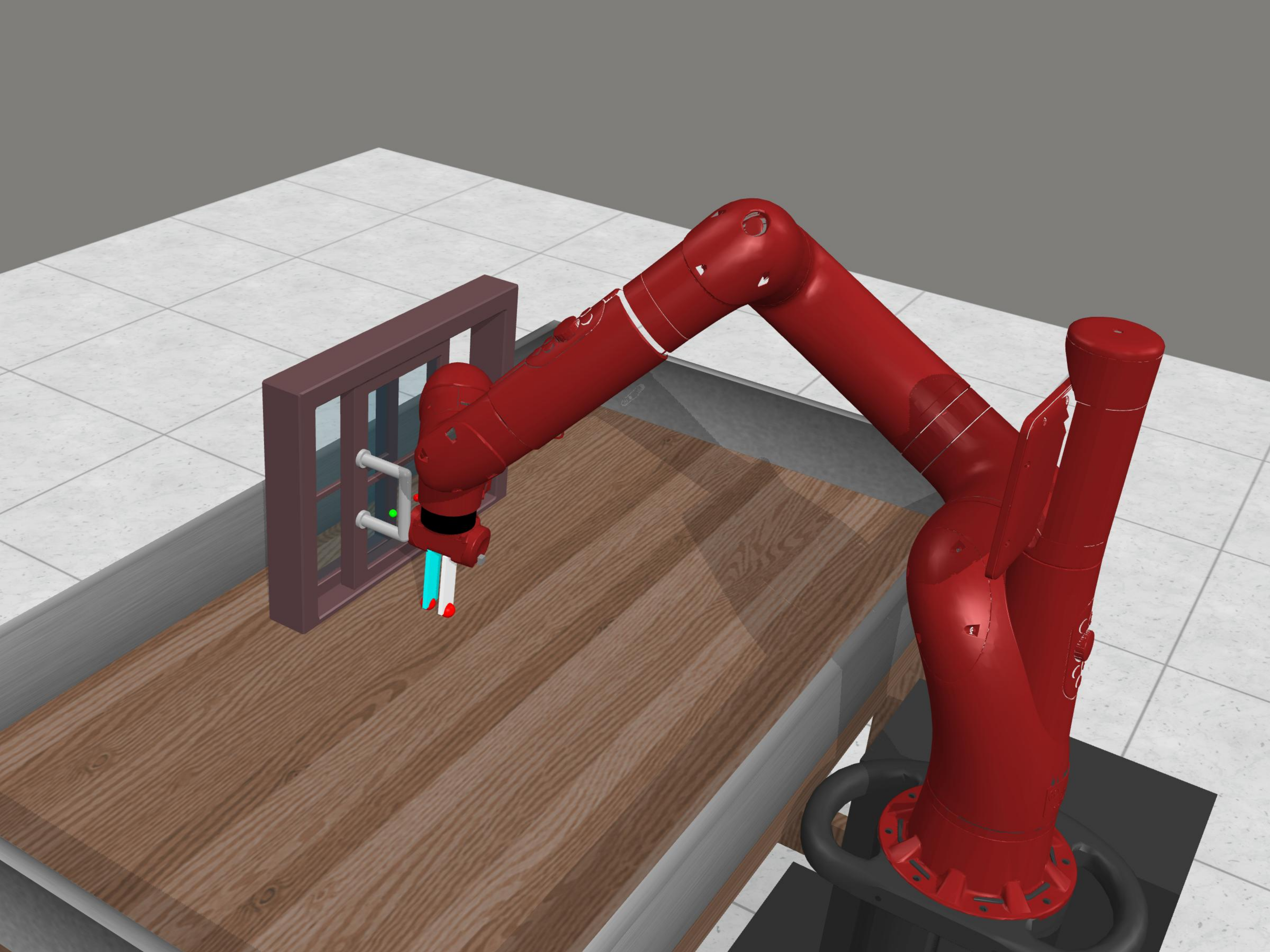}}
&\subfloat[Sweep Into]{\includegraphics[width=0.27\linewidth, height=0.18\linewidth]{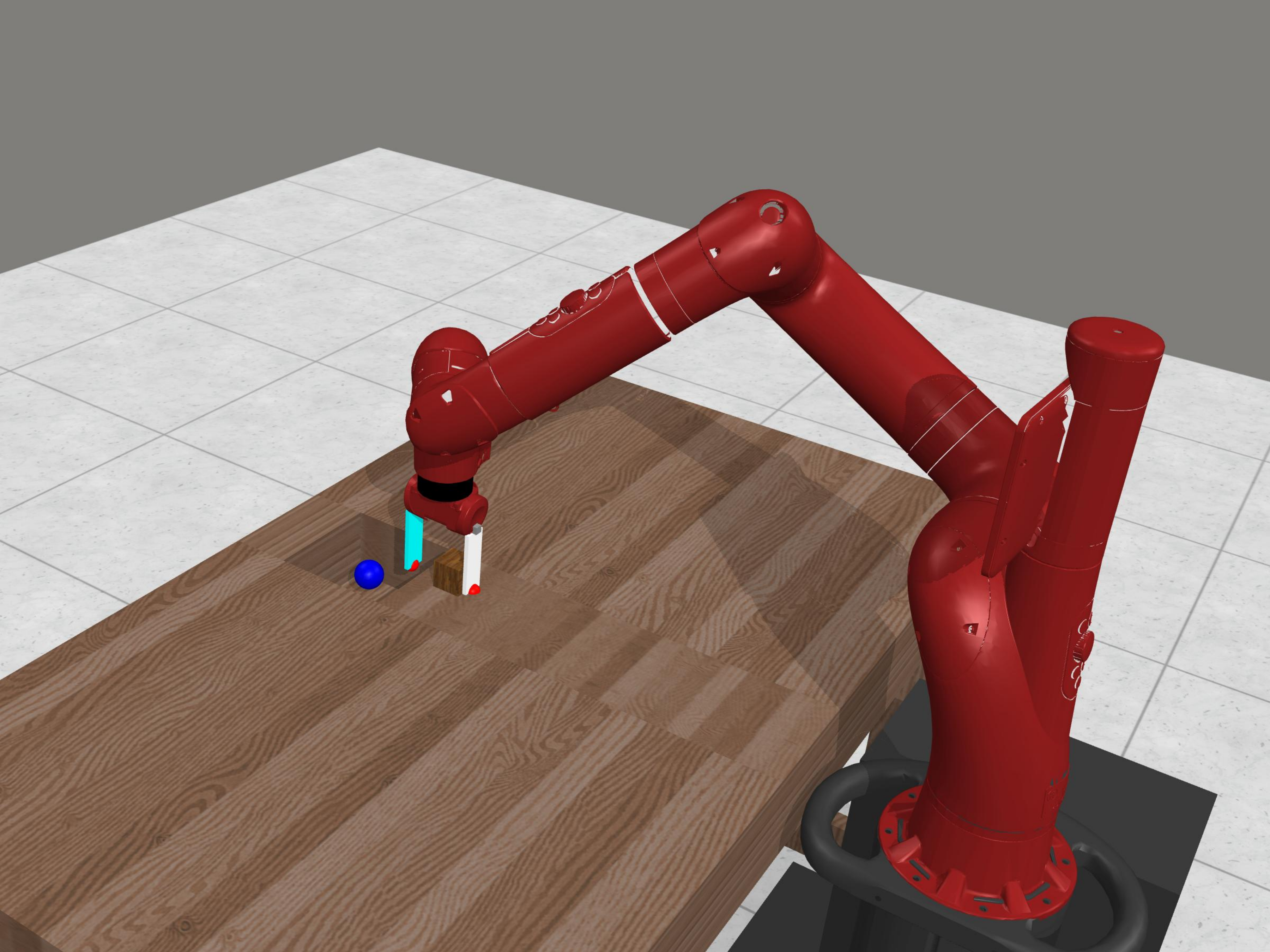}}
\end{tabular}
\caption{Six tasks are used for experiments. (a-c) DMControl tasks. (d-f) Meta-world tasks.}
\label{fig:render_tasks}
\end{figure*}

\subsection{D4RL benchmark}
\textbf{AntMaze}.
AntMaze involves a navigation challenge where a Mujoco Ant robot must locate and reach a specified goal. The data for this task is derived from a pre-trained policy designed to navigate various maze layouts, primarily focusing on two configurations: \emph{medium} and \emph{large}. The data are compiled using two distinct strategies: \emph{diverse} and \emph{play}. \emph{Diverse} data sets are produced by the pre-trained policy, which utilizes random start and goal locations, whereas \emph{play} data sets are generated with specific, intentionally selected goal locations. The task's reward structure is sparse, awarding points only when the robot is within a predefined proximity to the goal; otherwise, no reward is granted.

\textbf{Gym-Mujoco Locomotion}.
In Gym-Mujoco locomotion tasks, the objective is to manage simulated robots (Walker2d, Hopper) to advance forward while minimizing energy expenditure (action norm) to ensure safe behavior. Two data generation methodologies are employed: \emph{medium-expert} and \emph{medium-replay}. The {emph{medium-expert} data sets have an equal mix of expert demonstrations and suboptimal (partially-trained) demonstrations. The {emph{medium-replay} data sets come from the replay buffer of a partially-trained policy. The robot's forward velocity, a control penalty, and a survival bonus all contribute to determining the task's reward.

\textbf{Robosuite Robotic Manipulation}.
In Robosuite's robotic manipulation tasks \citep{zhu2020robosuite}, various 7-DoF simulated hand robots perform distinct tasks. Our experiments utilize environments simulated with Panda by Franka Emika, focusing on two specific tasks: lifting a cube (\emph{lift}) and relocating a coke can from a table to a designated bin (\emph{can}). Data collection involves inputs from either a single proficient teleoperator (\emph{ph}) or six teleoperators with varying skill levels (\emph{mh}). The task's reward is sparse, with further details deferred to the original paper.

\section{Additional Experiments}
% \subsection{Result Summary}
% \input{tables/summary_1}
\subsection{Ablation Study}
\noindent \textbf{Impact on Preference Amounts.} To evaluate how the quantity of preferences influences the performance of \ourmethod, we carry out an additional experiment to evaluate SEER's efficacy with varying amounts of human preferences. We consider the number of preferences $N \in \{50,100,300,500 \}$ on Sokoban tasks. The training curves of the average episode return of all methods on tasks are in Figure~\ref{fig:main_abla_frequence}. The results indicate that the performance of the policy gradually improves as the number of preference labels increases. When provided with sufficient preference labels, \ourmethod can approach the performance upper bound. The learning curves depicting the average episode return for all methods across tasks are presented in Figure~\ref{fig:main_abla_frequence}. The results suggest that as the number of preference labels increases, there's a corresponding improvement in the policy's performance. As the provision of preference labels increases, \ourmethod is capable of approaching optimal performance for each task.
\begin{figure*}[!ht]
\centering
\begin{center}
  \includegraphics[width=0.94\linewidth]{figures/results/main_header.pdf}
  \vspace{-1.2em}
\end{center}
\begin{tabular}{ccccc}
\hspace*{-2.4em}
\rotatebox{90}{\qquad\qquad \ \ Push-5x5-1}
& \hspace*{-1.25em} \subfloat[feedback=50]{\includegraphics[width=0.23\linewidth]{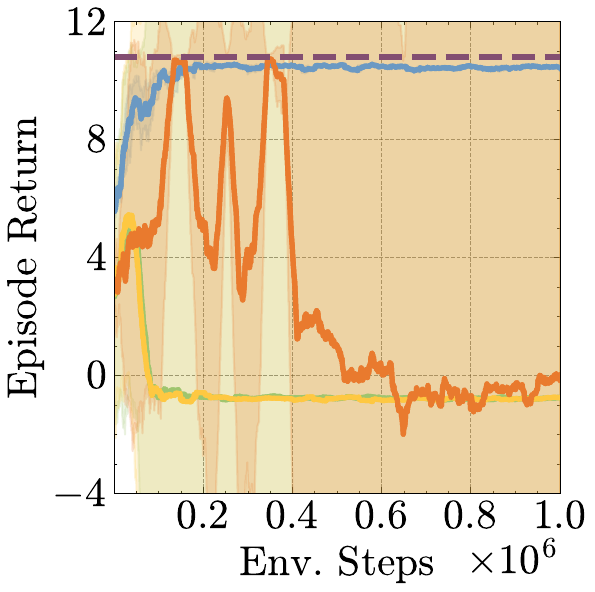}}
& \hspace*{-1.65em} \subfloat[feedback=100]{\includegraphics[width=0.23\linewidth]{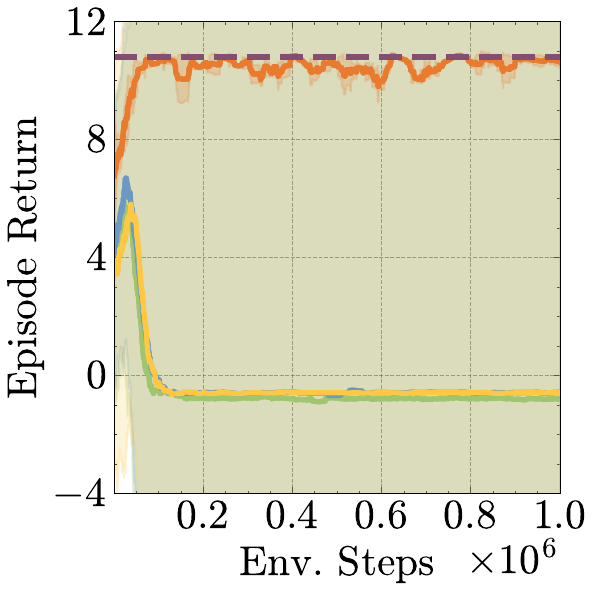}}
& \hspace*{-1.65em} \subfloat[feedback=300]{\includegraphics[width=0.23\linewidth]{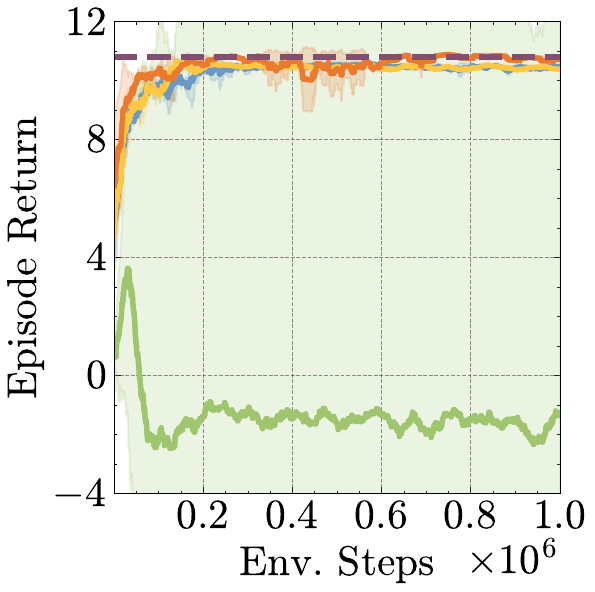}}
& \hspace*{-1.65em} \subfloat[feedback=500]{\includegraphics[width=0.23\linewidth]{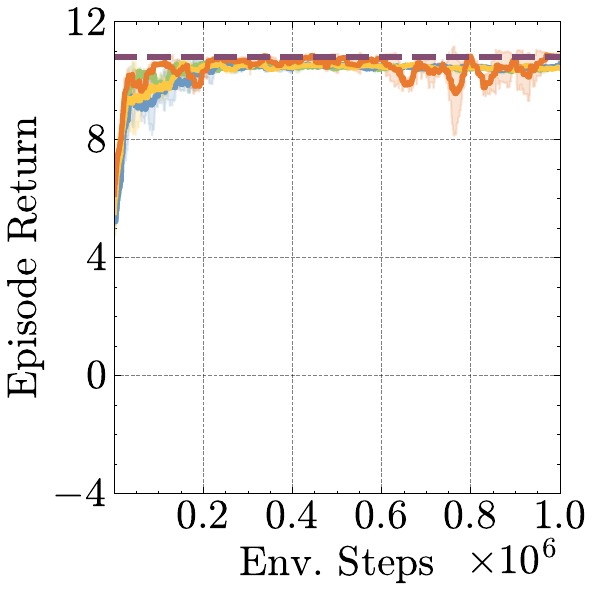}}
\end{tabular}
\caption{Training curves of all methods with varying numbers of preference labels on Push-5x5-1. The solid line presents the mean values, and the shaded area denotes the standard deviations over five runs.}
\label{fig:main_abla_frequence}
\end{figure*}

\noindent \textbf{Computational Efficiency of \ourmethod.}
When developing algorithms for improved performance and sample efficiency, it is also important to consider the training time overhead in real-world applications. In discrete settings, vertices are stored in a dictionary using their hashes as keys, enabling state retrieval in $\mathcal{O}(1)$ time. Moreover, updates are only required at the end of each episode. In continuous settings, \ourmethod allows the actor to update at a lower frequency, thus not significantly increasing computational costs. To quantitatively demonstrate the time overhead of all methods, we summarize the training time (hour) of \ourmethod and baseline methods on Push-7x7-2 and Cheetah Run in the table below. As shown in Table \ref{tab:time_cost}, compared to the backbone algorithm PEBBLE, \ourmethod does not significantly increase training time overhead, whereas SURF incurs a noticeable increase in time due to the need to label samples.

{
\setlength{\extrarowheight}{1.2pt}
\begin{table}[htbp]
\centering
\caption{Training time (hour) of various tasks.}
\label{tab:time_cost}
\begin{tabular}{ccccc}
\hline
Task/Methods & PEBBLE & SURF & MRN  & \ourmethod (ours) \\ \hline
Push-7x7-2   & 2.78   & 3.56 (+28.1\%) & 3.11 (+11.9\%) & 2.84 (\textbf{+2.2\%})\\
Cheetah Run  & 3.23   & 3.83 (+18.6\%) & 3.42 (+5.9\%)  & 3.31 (\textbf{+2.5\%})\\ \hline
\end{tabular}
\end{table}
}

\noindent \textbf{Sensitivity Analysis of the Parameter $\eta$.}
Policy regularization is an important technique in \ourmethod. 
In our experiments, we did a parameter search for $\eta$ and chose a robust value in the middle of the search space. The table below displays the performance of \ourmethod on Cheetah Run with various values of $\eta$. In all our experiments conducted in online settings, we set $\eta$ to 6, as shown in Table \ref{table:hyperparameters_pebble}.
{
\setlength{\extrarowheight}{1.2pt}
\begin{table}[htbp]
\centering
\caption{Parameter search for policy regularization.}
\label{tab:time_cost}
\begin{tabular}{cccccc}
\hline
$\eta$      & 0    & 0.1  & 1.0  & 5.0  & 10.0 \\ \hline
Cheetah Run &542.1 & 579.7&	697.3&710.7 &713.5 \\ \hline
\end{tabular}
\end{table}
}

\subsection{Human Experiments}
We demonstrate that agents can perform various novel behaviors based on human feedback using \ourmethod, as shown in Figure \ref{fig:human}. Specifically, we illustrate a cheetah running solely on its hind legs, guided by 200 preference labels. In offline settings, all tasks utilize feedback provided by a human instructor.

\begin{figure}[htbp]
    \centering
    \includegraphics[width=0.9\linewidth]{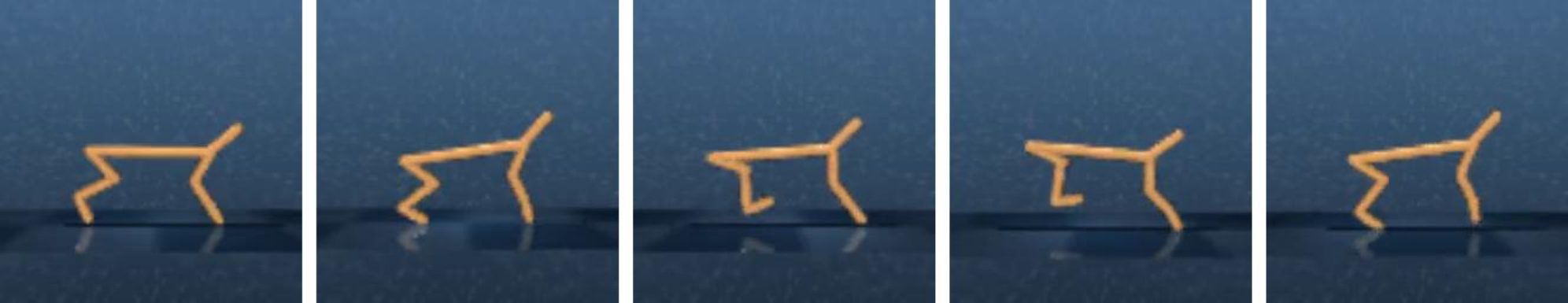}
    \caption{A cheetah running solely on its hind legs.}
    \label{fig:human}
\end{figure}

\section{Full Experimental results}
We present all the results of the discrete tasks.
\begin{figure*}[!ht]
\centering
\begin{center}
  \includegraphics[width=0.95\linewidth]{figures/results/main_header.pdf}
  \vspace{-1.2em}
\end{center}
\begin{tabular}{ccc}
\hspace*{-0.8em}
  \hspace*{-1.1em} \subfloat[Push-5x5-1 (feedback=300) ]{\includegraphics[width=0.31\linewidth]{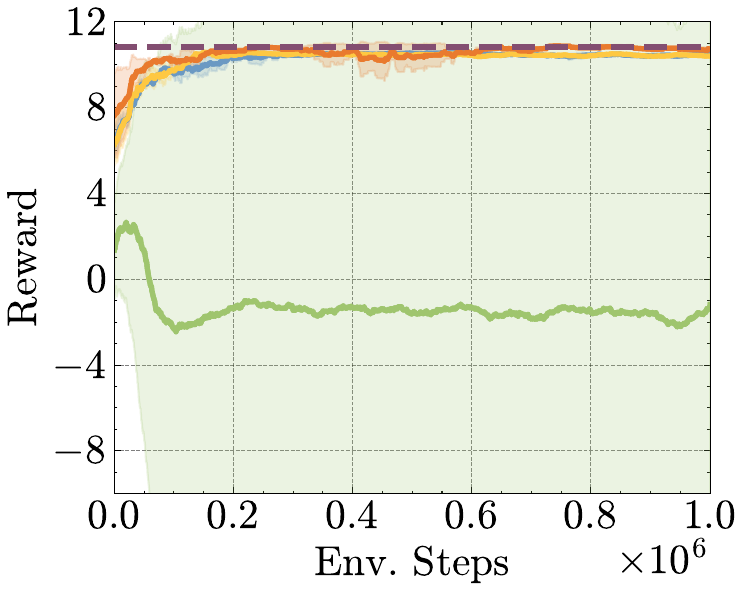}}
& \hspace*{-1.1em} \subfloat[Push-6x6-1 (feedback=300) ]{\includegraphics[width=0.31\linewidth]{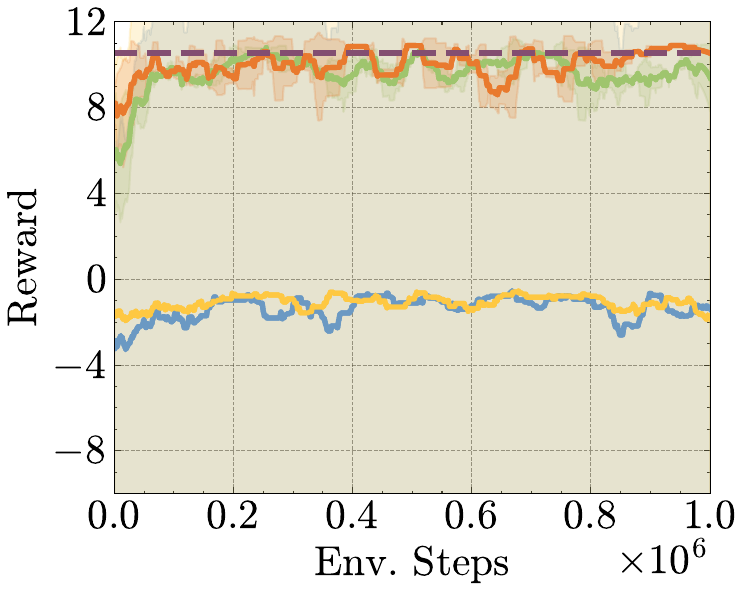}}
& \hspace*{-1.1em} \subfloat[Push-7x7-1 (feedback=300) ]{\includegraphics[width=0.31\linewidth]{figures/results/discrete/Sokoban771_300.pdf}}
\vspace*{-0.5em} \\
\hspace*{-0.8em}
  \hspace*{-1.1em} \subfloat[Push-5x5-2 (feedback=1000)]{\includegraphics[width=0.31\linewidth]{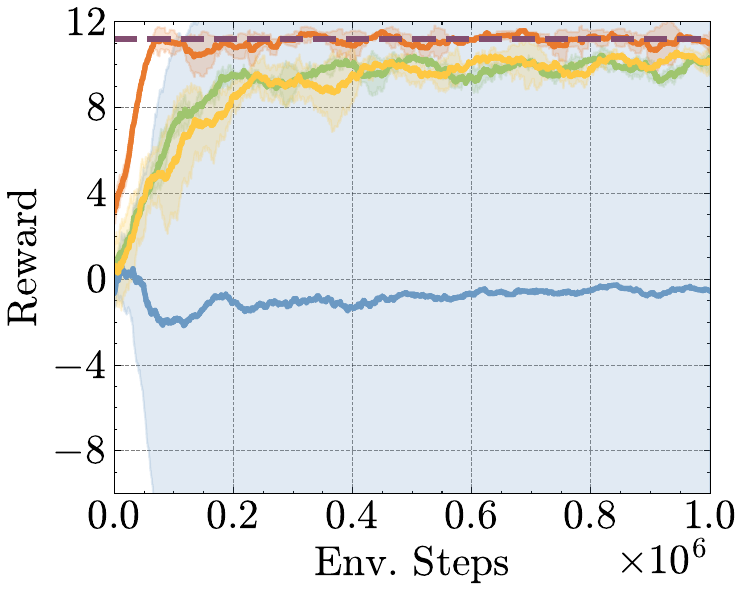}}
& \hspace*{-1.1em} \subfloat[Push-6x6-2 (feedback=1000)]{\includegraphics[width=0.31\linewidth]{figures/results/discrete/Sokoban662_1000.pdf}}
& \hspace*{-1.1em} \subfloat[Push-7x7-2 (feedback=1000)]{\includegraphics[width=0.31\linewidth]{figures/results/discrete/Sokoban772_1000.pdf}}
\end{tabular}
\caption{Training curves of various methods on six puzzle-solving tasks from Sokoban. The solid line presents the mean values, and the shaded area denotes the standard deviations over five runs. }
\label{fig:main_result_sokoban}
\end{figure*}

\section{Impact Statements}
Preference-based reinforcement learning allows us to tackle a wide range of problems and achieve robust AI without the need for reward engineering. However, it also poses potential risks, such as a malicious user manipulating preferences to instill harmful behaviors in the agent. Our method enhances the efficiency of preference-based RL, thereby simplifying the process of teaching both desirable and undesirable behaviors. When developing real-world applications, in addition to performance, safety is also a crucial consideration.

% \section{Appendix / supplemental material}

% Optionally include supplemental material (complete proofs, additional experiments and plots) in appendix.
% All such materials \textbf{SHOULD be included in the main submission.}

%%%%%%%%%%%%%%%%%%%%%%%%%%%%%%%%%%%%%%%%%%%%%%%%%%%%%%%%%%%%
\end{document}